\documentclass[10pt,leqno]{amsart}
\usepackage{latexsym}
\usepackage{amssymb}
\usepackage{amsmath}
\usepackage{amsthm}
\usepackage{booktabs}
\usepackage{enumitem}
\usepackage{graphicx}
\usepackage{color}
\usepackage{appendix}
\usepackage{mathrsfs} % manuscript style in math mode

\usepackage{soul}
\usepackage{url}
\usepackage[hidelinks]{hyperref}
\usepackage[utf8]{inputenc}
\usepackage[small]{caption}
\usepackage{booktabs}
\usepackage{xfrac}
\usepackage{comment}
\usepackage{stmaryrd}
\usepackage{tikz} 
\usetikzlibrary{shapes.misc,automata,positioning,arrows}
\usepgflibrary{shapes.arrows}

\usepackage{bbm} % To get |N, |R, Z, use \mathbbm{N}, \mathbbm{R}, \mathbbm{Z}
\usepackage{amstext}
\usepackage{mathtools}
\usepackage{mathrsfs} % manuscript style in math mode
\usepackage{ifsym}
\usepackage{subcaption}

\usepackage{algorithm}
\usepackage{algpseudocode}

\usepackage{tabularx}
\usepackage{multirow}

\newtheorem{theorem}{Theorem}

\newtheorem{proposition}[theorem]{Proposition}

\newtheorem{example}{Example}%[section]%}
\newtheorem{remark}{Remark}%[section]%}
\newcommand{\ie}{\mbox{i.e.}}
\newcommand{\eg}{\mbox{e.g.}}
\newcommand{\cf}{\mbox{cf.}}
\newcommand{\viz}{\mbox{viz.}}
\newcommand{\wrt}{\mbox{w.r.t.}}

\newcommand{\tuple}[1]{\langle #1 \rangle}

\newcommand{\Lang}{\ensuremath{\mathcal{L}}}

\newcommand{\KB}{\ensuremath{\mathcal{K}}} % a knowledge base
\newcommand{\D}{\ensuremath{\mathcal{D}}}
\newcommand{\calS}{\ensuremath{\mathcal{S}}}
\newcommand{\limp}{\rightarrow} % logical implication
\newcommand{\liff}{\leftrightarrow} % biconditional
\newcommand{\lxor}{\veebar}

\newcommand{\sat}{\Vdash}

\newcommand{\twiddle}{\hspace{-0.04cm}\mathrel|\joinrel\sim\hspace{-0.05cm}}
\newcommand{\calL}{\mathcal{L}}

\newcommand{\KBi}[1]{\ensuremath{\KB^{<\infty}}}
\newcommand{\KBiAB}[1]{\ensuremath{\KB^{A\twiddle\lnot B\mid <\infty }}}
\newcommand{\ii}[1]{\mbox{$(#1)$}}

\newcommand{\nd}{\noindent}

\newcommand{\defined}{\:\raisebox{1ex}{\scalebox{0.5}{\ensuremath{\mathrm{def}}}}\hskip-1.65ex\raisebox{-0.1ex}{\ensuremath{=}}\:}
\renewcommand{\vec}[1]{\mathbf{#1}}
\newcommand{\loggnn}{{\sc {MILP-SAT}-{GNN}}}

\newcommand{\realn}{\mathbb{R}}
\newcommand{\integern}{\mathbb{Z}}
\newcommand{\naturaln}{\mathbb{N}}

\newcommand{\satindicator}{ \mathbb{I}_{\mathtt{SAT}}}

\newcommand{\aggrf}{\mathtt{agg}}
\newcommand{\combf}{\mathtt{comb}}
\newcommand{\poolf}{\mathtt{pool}}

\newcommand{\multiset}[1]{\{\!\!\{ #1 \}\!\!\}}
\newcommand{\transp}[1]{{#1}^\intercal}

\newcommand{\calF}{\mathcal{F}}

\begin{document}
\title{\loggnn: Yet Another Neural SAT Solver} 
\author{%
Franco Alberto Cardillo$^1$ \and
Hamza Khyari$^2$ \and
Umberto Straccia$^3$ \\ \\
 $^1$CNR-ILC, Pisa, Italy \\
 $^2$USMBA-LSIA, Fez, Morocco \\
 $^3$CNR-ISTI, Pisa, Italy
}

\begin{abstract}
We proposes a novel method that enables Graph Neural Networks (GNNs) to solve SAT problems by leveraging a technique developed for applying GNNs to Mixed Integer Linear Programming (MILP). Specifically, k-CNF formulae are mapped into MILP problems, which are then encoded as weighted bipartite graphs and subsequently fed into a GNN for training and testing. From a theoretical perspective: (i) we establish permutation and equivalence invariance results, demonstrating that the method produces outputs that are stable under reordering of clauses and variables; (ii) we identify a theoretical limitation, showing that for a class of formulae called foldable formulae, standard GNNs cannot always distinguish satisfiable from unsatisfiable instances; (iii) we prove a universal approximation theorem, establishing that with Random Node Initialization (RNI), the method can approximate SAT solving to arbitrary precision on finite datasets—that is, the GNN becomes approximately sound and complete on such datasets. Furthermore, we show that for unfoldable formulae, the same approximation guarantee can be achieved without the need for RNI. Finally, we conduct an experimental evaluation of our approach, which show that, despite the simplicity of the neural architecture, the method achieves promising results. 

\end{abstract}

\maketitle

%%%%%%%%%%%%%%%%%%%%%%%%%%%%%%%%%%%%%%%%%%%%%%%%%%%%%%%%%%%%%%%%%%%%%%%%

%%%%%%%%%%%%%%%%%%%%%%%%%%%%%%%%%%%%%%%%%%%%%%%%%%%%%%%%%%%%
\section{Introduction}
%%%%%%%%%%%%%%%%%%%%%%%%%%%%%%%%%%%%%%%%%%%%%%%%%%%%%%%%%%%%

The \emph{Boolean Satisfiability Problem} (SAT), and particularly its restricted variant 3-SAT (see definition later on), stands as a fundamental challenge in computer science, underpinning numerous applications across formal verification, artificial intelligence, theory of computation, and optimization (see, \eg~\cite{Gomes08}). While decades of research have yielded highly sophisticated SAT solvers based on conflict-driven clause learning, heuristics, and systematic search, these solvers often face  difficulties when tackling large-scale, structurally complex, or adversarially generated instances (\eg~designing heuristics remains a highly
non-trivial and time-consuming task).
In parallel, the field of machine learning, and in particular  \emph{Graph Neural Networks} (GNNs)~\cite{Merkwirth05,Scarselli09}, has demonstrated remarkable capacity to model and reason about structured, relational data, 
offering a promising opportunity to enhance or even replace modern SAT solvers. 
However, applying GNNs directly to SAT presents some challenges:
SAT formulae have rich, logical structure that is difficult to fully capture with standard GNN architectures. Also, GNNs, due to their inherent limitations may fail to distinguish between critical differences in satisfiability status. Furthermore, achieving soundness (if the solver's decision is SAT then the satisfiability status is indeed SAT) and completeness (if the satisfiability status is SAT then the solver's decision is SAT) is particularly difficult in learned settings.

{\bf Related Work.} The use of GNNs for SAT solving is not novel as documented by the GNN SAT solvers that have been proposed in the literature 
%such as~\cite{Selsam19,Selsam19b,Ozolins22,Li22,Amizadeh19,Shi23} 
(see~\cite{Guo23,Holden21,Li24}, for an overview). 
Existing GNN-based SAT solvers can be broadly categorized into two
branches: \emph{stand-alone neural solvers} and \emph{neural-guided solvers}.
Stand-alone neural solvers utilize GNNs to solve SAT instances directly.
For instance, \cite{Selsam19,Cameron20,Shi23} focuses on predicting the satisfiability of a given formula, while several alternative approaches~\cite{Amizadeh19b,Ozolins22,Li23}
aim to construct a satisfying assignment.
Neural-guided solvers, on the other hand, integrate GNNs with modern SAT solvers, trying to improve \eg~their search heuristics  with the prediction of GNNs.  In this category fall methods such as~\cite{Li22,Selsam19b,Yan24,Yolcu19,Zhang20}.  Eventually, in~\cite{Li24} a benchmark dataset has been proposed to evaluate some of these GNN models.
%
%\vspace*{-3ex}
{\bf Contribution.} Our contribution falls in the former stand-alone neural solvers  category and is as follows.  We propose a new method that enables GNNs to solve SAT problems by leveraging a technique developed for applying GNNs to Mixed Integer Linear Programming (MILP). Specifically, k-CNF formulae are mapped into MILP problems, which are then encoded as weighted bipartite graphs and subsequently fed into a GNN for training and testing. 
We believe that, having an underlying MILP representation, may allow us, at least in theory (with appropriate and more involved transformations), to deal with some \emph{many-valued}~\cite{Haehnle94a}, paraconsistent~\cite{Abe15}, probabilistic and/or logics combining logical constraints with numerical constraints as well~\cite{LukasiewiczT98,Martires24,deRaedt07}, opening up to new perspectives of GNNs usages.

From a theoretical perspective:
\ii{i} we establish permutation and equivalence invariance results, demonstrating that the method produces outputs that are stable under reordering of clauses and variables; 
\ii{ii} we prove a universal approximation theorem, establishing that with \emph{Random Node Initialization}  (RNI)~\cite{Abboud20,Grohe21,Sato21}, the method can approximate SAT solving to arbitrary precision on finite datasets—that is, the GNN becomes approximately sound and complete on such datasets. Furthermore, we show that for unfoldable formulae, the same approximation guarantee can be achieved without the need for RNI; and
\ii{iii} finally, we have implemented our method and have conducted an experimental evaluation, demonstrating
that, despite the simplicity of the neural architecture, the method achieves promising results.

In the following, we proceed as follows. In the next section, we introduce the background notions we will rely on. In Section~\ref{loggnn} we illustrate our method and theoretical results, while in Section~\ref{eval} we address the experimental validation. Section~\ref{concl} summarise our contribution and illustrates topics for future work.

%%%%%%%%%%%%%%%%%%%%%%%%%%%%%%%%%%%%%%%%%%%%%%%%%%%%%%%%%%%%
\section{Background} \label{bckg}
%%%%%%%%%%%%%%%%%%%%%%%%%%%%%%%%%%%%%%%%%%%%%%%%%%%%%%%%%%%%

\nd For readers not be familiar with both CNFs and GNNs, we provide here a succinct summary of the basic notions we rely on. The informed reader may just skip a part and come back if needed.

%%%%%%%%%%%%%%%%%%%%%%%%%%%%%%%%%%%%%%%%%%%%%%%%%%%%%%%%%%%%
%\subsection{Propositional Logic} \label{pl}
{\bf CNFs.} \label{pl}
%%%%%%%%%%%%%%%%%%%%%%%%%%%%%%%%%%%%%%%%%%%%%%%%%%%%%%%%%%%%
%
%(all symbols may have an optional sup-, or sub-script or apex)
\nd Let $\Sigma=\{p_1, p_2, \ldots, p_{n}\}$ be a finite non-empty set of $n$ propositional \emph{letters} (or \emph{variables}).
% We define $n_\Sigma = |\Sigma|$ as the number of letters occurring in $\Sigma$ and may write simply $n$ if clear from context.
%
Let $\Lang_\Sigma = \{\alpha, \beta,\ldots\}$ be the set of propositional formulae based on the set of Boolean operators  
$\{\neg, \land, \lor \}$ and $\Sigma$ inductively defined as follows: \ii{i} every propositional letter is a formula; and \ii{ii} if $\alpha, \beta$ are formulae, so are $\neg \alpha, \alpha \land \beta$ and $\alpha \lor \beta$.
The Boolean operator $\neg$ (resp.~$\land, \lor,  \limp, \liff$ is called the \emph{negation} (resp.~\emph{conjunction, disjunction}) operator.
A \emph{literal} (denoted $l$) is either a propositional letter $p_i$ (called \emph{positive literal}) or its negation $\neg p_i$ (called \emph{negative literal}) and with $\bar{l}$ we denote $\neg p_i$ (resp.~$p_i$) if $l=p_i$ (resp.~if $l=\neg p_i$). 
%
% A formula is in \emph{Negation Normal Form} (NNF) if a negation operator $\neg$ is only applied to variables and the only other allowed Boolean operators are conjunction and disjunction.
%
A \emph{clause} is a (finite) disjunction of literals, while a \emph{$k$-clause} is a clause of $k$ literals ($1< k \leq n$). In a clause no letter may occur more than once. A formula is in \emph{Conjunctive Normal Form} (CNF) if it is a (finite) conjunction of clauses. It is  a $k$-CNF formula if it is a (finite) conjunction of $k$-clauses. All clauses in a CNF formula are assumed to be  distinct. 
% A $2$-clause is called a \emph{Krom clause} and a $2$-CNF formula is called a \emph{Krom formula}. 
% A $k$-clause where at most one of the literals is positive is called \emph{Horn clause} and a $k$-CNF formula in which all clauses are Horn is called a \emph{Horn formula}.
A $k$-CNF formula $\phi$ of $m$ clauses 
$\phi  =  \bigwedge_{i=1}^m C_i$,  with  $C_i  =  \bigvee_{j=1}^{k} l_{ij}$
%
% \begin{eqnarray}\label{kcnff}
%     \phi & = & \bigwedge_{i=1}^m C_i , \text{ with } C_i  =  \bigvee_{j=1}^{k} l_{ij}
% \end{eqnarray}
can also be represented as a \emph{set of sets},  by representing each $k$-clause as a set of literals and the conjunction of clauses as the set of clauses represented as sets. That is, 
$\phi  =  \{C_1, \ldots, C_m \}$,  with  $C_i  =  \{l_{i1}, \ldots , l_{ik} \}$.
%
%
% \begin{eqnarray} \label{kcnfset}
%     \phi & = & \{C_1, \ldots, C_m \} \text{ with } C_i  =  \{l_{i1}, \ldots , l_{ik} \}  \ .
% \end{eqnarray}
%
%In what follows, if not stated differently, we will always assume that a $k$-CNF formula is represented as a set of sets. % of the form (\ref{kcnfset}).
%
% \nd If a clause is represented as set, we call it an \emph{set clause}. A formulae of the form Eq.~\ref{kcnfset} is called a \emph{set $k$-CNF formula}. 
%
Note that there are at most $k_{\mathtt{CNF}} =  2^k \binom{n}{k}$
% \begin{equation} \label{distkcnf}
%     k_{\mathtt{CNF}} =  2^k \binom{n}{k}
% \end{equation}
distinct $k$-clauses and, thus, any  $k$-CNF formula can not have more than $k_{\mathtt{CNF}}$ $k$-clauses.
With $\calL^k_{\mathtt{CNF}}$ we denote the set of all $k$-CNF formulae, while with $\calL^k_{\mathtt{CNF}}(n,m)$ we denote $k$-CNF formulae having at most $n$ variables and $m$ clauses.
%
% \nd can also be represented as a \emph{set of sets}, denoted $\bar{\phi}$, by representing each $k$-clause $C_i$ as a set $\bar{C}_i$ of literals and the conjunction of clauses as the set of clauses represented as sets. That is, 
% \begin{eqnarray} \label{kcnfset}
%     \bar{\phi} & = & \{\bar{C}_1, \ldots, \bar{C}_m \} \text{ with } \bar{C}_i  =  \{l_{i1}, \ldots , l_{ik} \}  \ .
% \end{eqnarray}
%
% \nd If a clause is represented as set, we call it an \emph{set clause}. A formulae of the form Eq.~\ref{kcnfset} is called a \emph{set $k$-CNF formula}. 
% %
% Note that there are at most
% \begin{equation} \label{distkcnf}
%     k_{\mathtt{CNF}} =  2^k \binom{n}{k}
% \end{equation}
%
% \nd distinct set $k$-clauses and, thus, any set $k$-CNF formula can not have more than $k_{\mathtt{CNF}}$ set clauses.
%
% In what follows, if not stated differently, we will always assume that a $k-CNF$ $\phi$ is represented as a set of sets of the form (\ref{distkcnf}).
%
% With $\calL^k_{\mathtt{CNF}}$ (resp.~ $\bar{\calL}^k_{\mathtt{CNF}}$) we denote the set of all $k$-CNF (resp.~set $k$-CNF) formulae.
%
% A \emph{knowledge base} (KB) $\KB = \{\phi_1, \ldots, \phi_n\}$ is a finite set of formulae $\phi_i$ and with $\lbigand \KB$ we denote the formula $\lbigand_{\phi \in \KB} \phi $. In the following, whenever we write $\KB$, we consider $\lbigand \KB$ instead, unless stated otherwise.
%
Given a formula $\phi$, with $\Sigma_\phi \subseteq \Sigma$ we denote the set of propositional letters occurring in $\phi$. %(we may omit the reference to $\Sigma$ if no ambiguity arises).
%
% We define the \emph{size}  of $\phi$, denoted $|\phi|$, inductively as usual: for $p_i \in \Sigma$, $|p_i|=1$, $\neg \phi = 1 + |\phi|$, $|\phi \lor \psi| = |\phi \land \psi| = |\phi \limpf \psi| = |\phi \liff \psi| = 1 + |\phi| + |\psi|$.
%
An \emph{interpretation}, or \emph{world},  $w$ \wrt~$\Sigma$ is a set of literals such that all propositional letters in $\Sigma$ occur exactly once in $w$.
%
%$\W_\Sigma$ is the set of all worlds \wrt~$\Sigma$.  
%
We may also denote a world $w$ as the concatenation of the literals occurring in $w$ by replacing a negative literal $\neg p$ with $\bar{p}$ (\eg~$\{\neg p_1, p_2\}$ may be denoted also as $\bar{p}_1p_2$). 
With $w\sat \phi$ we indicate that the world $w$ \emph{satisfies} the formula $\phi$, \ie~$w$ is a \emph{model} of $\phi$, which is inductively defined as  usual:
 $w\sat  l$ iff $l \in w$, $w \sat \neg \phi$ iff 
 $w \not \sat \phi$, $w \sat  \phi \land \psi$ iff $w \sat  \phi$ and $w \sat  \psi$ and $w \sat  \phi \lor \psi$ iff $w \sat  \phi$ or $w \sat  \psi$. 
 %$w \sat  \phi \limpf \psi$ iff $w \sat  \neg \phi \lor \psi$, and eventually $w \sat  \phi \liff \psi$ iff $w \sat ( \phi \limpf \psi) \land (\psi \limpf \phi)$.  
%
Furthermore, $\phi$ is \emph{satisfiable} (resp.~\emph{unsatisfiable}) if it has (resp.~has no) model. The \emph{$k$-SAT} problem consists in deciding whether a $k$-CNF formula is satisfiable or not and is known to be a NP-complete problem for $k\geq 3$~\cite{Cook71}.\footnote{The $k$-SAT problem can be solved by a nondeterministic Turing machine in polynomial time.} 
%However, note that the 2-SAT problem is NL-complete~\cite{Papadimitriou94} instead.\footnote{The 2-SAT problem can be solved by a nondeterministic Turing machine using a logarithmic amount of memory space.}
%
% We will also use two special symbols $\remaind$ and $\top$:  $\remaind$ is the formula without models, while every world satisfies $\top$.  We impose that $\top$ (resp.~$\remaind$) cannot occur in any other formula different than $\top$ (resp.~$\remaind$) itself.
%
% Given formulae $\phi$ and $\psi$, let 
% $[\phi]_{\Sigma}$ be the set of models of $\phi$ \wrt~$\Sigma$;
% we say the $\phi$ \emph{entails} $\psi$ (or also $\psi$ is a \emph{logical consequence} of $\phi)$, denoted $\phi \entails \psi$, iff $[\phi]_{\Sigma} \subseteq [\psi]_{\Sigma}$ (\ie~all models of $\phi$ are models of $\psi$), and $\phi$ and $\psi$ \emph{equivalent}, denoted $\phi \equiv \psi$, iff $[\phi]_{\Sigma} = [\psi]_{\Sigma}$ (\ie~$\phi$ and$\psi$ have the same set of models).
% Let us note that checking entailment can be reduced to satisfiability test via the well-known property $\phi \models \psi$ iff $\phi \land \neg \psi$ unsatisfiable.
%
A \emph{dataset} is a finite set of formulae-label pairs ($s \geq 2$)
$\D = \{\tuple{\phi_1,y_1}, \ldots, \tuple{\phi_s, y_s}\}$,
% \begin{equation} \label{dataset}
%     \D = \{\tuple{\phi_1,y_1}, \ldots, \tuple{\phi_s, y_s}\} \ ,
% \end{equation}
such that $y_i \in\{0,1\}$ and $y_i = 1$ iff $\phi$ satisfiable.\footnote{Ideally, we would like to have that there is no pair involving equivalent formulae, \ie~$\phi_i \not \equiv \phi_j$ for each $1 \leq   i  < j \leq s$ and, thus, $|\D| \leq 2^{2^{n}}$. However, for large size dataset, \eg~in the order of $10^5$, producing such a dataset is unfeasible (time consuming) and, we are not going to consider this further constraint.}
A \emph{CNF} (resp.~$k$-CNF) \emph{dataset} is a dataset of CNF (resp.~$k$-CNF) formulae.
%
% We conclude this section by recalling the well-known fact that any formula can be transformed into an equivalence preserving $k$-CNF formula 
% via the transformations illustrated in Appendix~\ref{cnftransform} and that  
% a standard way to represent CNF formulae is to rely on the so-called \emph{DIMACS} format (see Appendix~\ref{dimacs}).
%
% Finally, we recall the well-known fact that any formula can be transformed into an equivalence preserving $k$-CNF formula,\footnote{See, \eg~\cite{vanHarmelen08}, or also \url{https://en.wikipedia.org/wiki/Tseytin_transformation}} 
% which in turn can be transformed into an equivalent $3$-CNF formula. 
% and that a standard way to represent $k$-CNF formulae is to rely on the so-called \emph{DIMACS} format.\footnote{This is the standard input format in the SAT competition conference \url{https://satcompetition.github.io}}

%%%%%%%%%%%%%%%%%%%%%%%%%%%%%%%%%%%%%%%%%%%%%%%%%%%%%%%%%%%%
%\subsection{Graph Neural Networks} \label{gnn}
{\bf Graph Neural Networks.} \label{gnn}
%%%%%%%%%%%%%%%%%%%%%%%%%%%%%%%%%%%%%%%%%%%%%%%%%%%%%%%%%%%%
%
% \nd \emph{Graph Neural Networks} (GNNs) are a class of deep earning
% models that operate on graphs (see, \eg~\cite{Hamilton20,Ju24}) for a recent overview).
%
\emph{Graph Neural Networks}~\cite{Merkwirth05,Scarselli09} (GNNs) are deep learning architectures for graph structured data (see, \eg~\cite{Hamilton20,Ju24}) for a recent overview). In general terms, each data point consists of the structure of a graph represented by the adjacency  matrix and properties of the graphs. The properties of these graphs are stored in node feature vectors, and can sometimes include edge feature vectors. 
%(see Figure~\ref{figgnn}). 

% \begin{figure}
% \centering
% \subcaptionbox{Example of GNN.\label{figgnn}}
% %[.4\textwidth]{\includegraphics[scale=0.5]{deep_neural-net.png}}
% {\includegraphics[scale=0.35]{gnn.pdf}}
% %
% \subcaptionbox{Example of MLP.\label{figmlp}}
% %[.4\textwidth]{\includegraphics[scale=0.5]{deep_neural-net.png}}
%  {\includegraphics[scale=0.35]{mlp.pdf}}
% \caption{Examples of a GNN and a MLP.}\label{figNN}
% \end{figure}

% \begin{figure*}
% \centering
% \begin{subcaptionblock}{.5\textwidth}
% \centering
% \includegraphics[scale=0.275]{gnn.pdf}
% \caption{Example of a GNN.}\label{figgnn}
% \end{subcaptionblock}%
% \begin{subcaptionblock}{.5\textwidth}
% \centering
% \includegraphics[scale=0.275]{mlp.pdf}
% \caption{Example of a MLP.}\label{figmlp}
% \end{subcaptionblock}%
% \vspace*{3ex}
% \caption{Example of a GNN and a MLP, each having three hidden layers.}\label{figNN}
% \vspace*{3ex}
% \end{figure*}

Formally, a \emph{graph} is a tuple $G= \tuple{V,E,\vec{X}}$, where $V = \{v_1, \ldots, v_{|V|}\}$ is the set of nodes, $E = \{e_1, \ldots, e_{|E|}\}$ is the set of edges, and  the edge $e = \tuple{v_i, v_j} \in E$ represents the connection relationship between nodes $v_i$ and $v_j$ in the graph. $\vec{X} \in \realn^{|V|\times M}$ is the node feature matrix with $M$ being the dimension of each node feature. With $\vec{x}_v$ we denote the feature vector of $v$.\footnote{Edges may also have an associated feature vector and being embedded, but for ease of representation we do not consider this case here.} The adjacency matrix of a graph can be defined as $\vec{A} \in \realn^{|V| \times |V|}$, where $\vec{A}_{ij} = 1$ if $\tuple{v_i, v_j} \in E$, otherwise $\vec{A}_{ij} = 0$. We use $N(v)$ to denote the \emph{neighbours} of a node $v$ and with $d_v$ its degree. The graph $G$ is \emph{undirected} if $\vec{A}$ is \emph{symmetric}, \ie~$\vec{A}_{ij} = \vec{A}_{ji}$ for all $1\leq i, j \leq n$. $G$ has a \emph{self-loop} if $\vec{A}_{ii} = 1$ for some $1\leq i\leq n$. The \emph{order} of a graph is the number of nodes in it.

% The adjacency matrix $\vec{A}$ corresponding to $G$ could be a highly sparse
% matrix, and if $\vec{A}$ is used directly as node representations, it will be seriously affected by impractical storage demands and computational overhead, besides making subsequent machine learning tasks very difficult. To overcome this issue, various methods of so-called \emph{graph embeddings}, \viz~\emph{node embeddings}, methods have been developed~\cite{Ju24}, that consist of techniques for dimensionality reduction. That is, the goal consists of  learning a embedding vector representation $\vec{h}_v \in \realn^d$ for each node $v\in V$, where the dimension $d$ of the vector is much smaller than the total number of nodes $|V|$ in the graph.

The adjacency matrix $\vec{A}$ of a graph $G$ is often highly sparse, making it inefficient for storage, computation, and machine learning tasks. To address this, \emph{graph embedding} techniques—specifically \emph{node embeddings}—have been developed~\cite{Ju24} to reduce dimensionality by learning a compact vector representation $\vec{h}_v \in \mathbb{R}^d$ for each node $v \in V$, where $d \ll |V|$.

Many different versions of the GNNs have been considered in the literature (see, \eg~\cite{Ju24}).
A somewhat standard and  core model is the one that goes under the name \emph{Message Passing Neural Network} (MPNN)~\cite{Gilmer17} or \emph{Aggregate-Combine Graph Neural Network} (AC-GNN)~\cite{Barcelo20}. AC-GNN are also called \emph{1-dimensional GNNs} in~\cite{Morris19}.

It consists of a sequence of layers that iteratively combine the representation vectors of every node with the multiset of representation vectors of its neighbours.  
Formally, consider a sequence $\tuple{L_1, ..., L_l}$ of layers $L_k$. $L_1$ is called the \emph{input layer}, $L_2, ..., L_{l-1}$ the \emph{hidden layers} and $L_l$ the \emph{output layer}. Let $\aggrf$ and $\combf$ be  \emph{aggregation} and \emph{combination} functions. 
The function $\aggrf$  aggregates the information from the neighbours of each node, while  $\combf$  updates the node representations by combining the aggregated information from neighbours with the current node representations. 
We assume that each node $v$ has an associated initial embedding vector 
$\vec{h}_v^1 = f_{\mathtt{in}}(\vec{x}_v)$, where $f_{\mathtt{in}}$ is a learnable function. 
A GNN computes a vector $\vec{h}_v^k$, called the node representation or \emph{node embedding} of $v$  at the $k$-th layer,  via the following recursive formula ($1\leq k < l$):\footnote{$\multiset{\cdot }$ denotes a multiset.}
\begin{eqnarray}\label{gnnform}
    %h_v^{k+1} & = & \combf(x_v,h_v^k, \aggrf(x_v,\multiset{ \tuple{h_w^k,x_w} \mid x_w \in N(v) })
    \vec{h}_v^{k+1} & = & \combf(\vec{h}_v^k, \aggrf(\multiset{\vec{h}_w^k \mid w \in N(v) })
\end{eqnarray}

\nd The node representations $h_v^l$ in the last layer can be treated as the final node representations.
Then, the vectors $h_v^l$ of the last layer $L_l$ are then \emph{pooled} together via

\begin{eqnarray}\label{gnnpool}
\vec{\hat{y}} &= & \poolf(\multiset{ \vec{h}_v^l \mid v \in V})
\end{eqnarray}

\nd to give a single graph vector $\vec{\hat{y}}$, the output of the GNN.\footnote{The $\aggrf$, $\combf$ and $\poolf$ functions may also take into account the feature vectors of the involved nodes, but for ease of presentation we do not consider it here.}

In order to translate the abstract GNN framework defined above  into something we can implement, we must give concrete instantiations to these update and aggregate functions. The most basic GNN framework, which is a simplification of the original
GNN models proposed in~\cite{Merkwirth05,Scarselli09}, is as follows ($1 \leq k < l$):

\begin{equation} \label{gnnbasic}
\begin{array}{rcl}
    %\vec{h}_v^1 & = & \vec{x}_v \\
    \vec{h}_v^1 & = & f_{\mathtt{in}}(\vec{x}_v) \\
    \vec{h}_v^{k+1} & = & \sigma ( \vec{W}^k_{self} \vec{h}_v^k  + \vec{W}^k_{neigh} \sum_{w\in N(v)} \vec{h}_w^k +   \vec{b}_v^k) \\
    \vec{\hat{y}}_v &= & f_{\mathtt{out}}(\vec{h}_v^{l}) \ , \text{ for all } v\in V \ ,
\end{array}
\end{equation}

\nd where $\vec{W}^k_{self}, \vec{W}^k_{neigh} \in \realn^{d \times d}$ are learnable weight matrices,
$f_{\mathtt{out}}$ is a learnable function and $\sigma$ is a non-linear \emph{activation function}. Commonly used activation functions include the sigmoid, tanh, and ReLU functions.\footnote{\url{https://en.wikipedia.org/wiki/Activation_function}.}
The bias term $\vec{b}_v^k \in \realn^{d}$ is often omitted for notional simplicity, but including the bias term can be important to improve effectiveness. 
Many more involved GNNs model haven been developed so far.
% , such as \emph{Graph Convolutional Network}~\cite{Kipf17} (GCN)~\cite{Kipf17}. 
We refer the reader to \eg~\cite{Hamilton20} for more insights on GNNs.

In the context of GNNs, two tasks are typically addressed. Informally, the \emph{node classification problem} is the task of predicting the label $y_v$  of each node $v$ in a graph,  while the \emph{graph classification problem}  is the task of predicting the label $y_G$ of a graph $G$, where a label may be a categorical class (binary classification or multiclass classification), or  a continuous number (regression). We anticipate that, in this work, we are considering the binary graph classification problem only in which the output of a GNN is always a value in $\{0,1\}$.\footnote{The output $1$ will denote that a formula is satisfiable, while $0$ will denote that the formula is unsatifiable.}
Specifically, a \emph{graph dataset} consists of  set of graph-label pairs ($s \geq 1$)
$\D = \{\tuple{G_1,y_1}, \ldots, \tuple{G_s, y_s}\}$,
% \begin{equation} \label{datasetG}
%     \D = \{\tuple{G_1,y_1}, \ldots, \tuple{G_s, y_s}\} \ ,
% \end{equation}
%
where $y_i \in\{0,1\}$. From such a dataset of labelled graphs (the \emph{training set}), $\tuple{G_i,y_i}$, with $y_i \in \{0,1\}$, we would like to induce (learn) a function $f_\mathtt{GNN}:\mathcal{G} \to \{0,1\}$ that given an unlabelled graph $G\in \mathcal{G}$ as input predicts its label $y_G  \in \{0,1\}$. To do so, in general one tries learn the parameters of the network, \eg~the weights and bias in Eq.~\ref{gnnbasic}, by minimising a \emph{loss function}. That is, to learn a function $f_\mathtt{GNN}$ from $\D$ such that it minimises the \emph{loss}
$L  =   \sum_{G_i \in \D} loss(\hat{y}_i, y_i)$,
% \begin{eqnarray} \label{loss}
%     L & = &  \sum_{G_i \in \D}^s loss(\hat{y}_i, y_i) \ ,
% \end{eqnarray}
%
where $\hat{y}_i$ is the output of $f_\mathtt{GNN}$ applied to $G_i$ and $y_i$ is the so-called \emph{the ground truth label} of $G_i$ and $loss$ quantifies the disagreement between $\hat{y}$ and $y$ (with $loss(\hat{y}, y) = 0$  if $\hat{y}= y$). 

% $L  =    \frac{1}{s}\sum_{G_i \in \D}^s (\hat{y}_i - y_i)^2$. Another popular loss function is the \emph{Binary Cross Entropy} (BCE)
% defined as $L =  \frac{1}{s}\sum_{G_i \in \D}^s (y_i \log \hat{y}_i )+ (1-y_i) \log (1- \hat{y}_i)$

% \begin{eqnarray} \label{lossMSE}
%     L & = &   \frac{1}{s}\sum_{G_i \in \D}^s 
%     %\| \mathtt{MLP}(\vec{\hat{y}}_{G_i}) - y_i \|_2^2 \ ,
%     (\hat{y}_i - y_i)^2 \ .
% \end{eqnarray}

% \nd where $\vec{\hat{y}}_{G_i} \in \realn^d$ is the graph-level output of a learned GNN applied to $G_i$ and $\mathtt{MLP}$ is a learned MLP with a univariate output applied to $\vec{\hat{y}}_{G_i}$, \ie~$\mathtt{MLP}(\vec{\hat{y}}_{G_i}) = \hat{y}_i \in \{0,1\}$ (here you may think of connecting the MLP directly to the last layer of the GNN, \cf~Eq.~\ref{gcn}).

% \nd For more insights on GNNs we refer the reader to \eg~\cite{Hamilton20,Wu22}.

%%%%%%%%%%%%%%%%%%%%%%%%%%%%%%%%%%%%%%%%%%%%%%%%%%%%%%%%%%%%
\section{MILP-based GNNs for $k$-SAT} \label{loggnn}
%%%%%%%%%%%%%%%%%%%%%%%%%%%%%%%%%%%%%%%%%%%%%%%%%%%%%%%%%%%%

\nd As anticipated, we leverage on the recently introduced method to learn GNNs that solve \emph{Mixed Integer Linear Programming} (MILP) problems~\cite{Chen23}\footnote{We refer the reader \eg~to~\cite{Zhang23} for an overview on methods using machine learning in the context of  MILP.} to decide whether a $k$-CNF formula is satisfiable or not. Specifically, k-CNF formulae are mapped into MILP problems, which are then encoded as weighted bipartite graphs and subsequently fed into a GNN for training and testing.

%-----------
\subsection{From $k$-CNFs to MILPs} \label{sekcnfcmilp}
%-----------

\nd We recall that a general MILP problem $M$ is defined as follows:% (see \eg~\cite{Avraamidou22}):

\begin{equation} \label{milp}
\begin{array}{l}
    \min \transp{\vec{c}} \vec{x} \text{ s.t. } 
    \vec{A} \vec{x} \geq \vec{b} \ , 
    \vec{l} \leq \vec{x} \leq \vec{u} \ , 
    x_j \in \integern, \forall j \in I \ , 
\end{array}
\end{equation}

\nd where
$\transp{\vec{c}} \vec{x} = \sum_{i=1}^n c_i\cdot x_i$ is the \emph{objective function} to be minimised;
$\vec{x} \in \realn^n$ are the \emph{variables};
$\vec{c} \in \realn^n$ are the \emph{coefficients} of the variables in the objective function;
$\vec{A} \in \realn^{m \times n}$ is a matrix of coefficients such that $\vec{A} \vec{x}$ is the set of linear \emph{constraints} bounded below by $\vec{b} \in \realn^m$, which is a vector of coefficients;
$\vec{l}, \vec{u} \in (\realn \cup \{-\infty, +\infty\})^n$ express lower and upper bounds of the variables; and
the index set $I \subseteq \{1, 2, \ldots, n\}$ denotes those indices $j$ where the  $x_j$ are constrained to be an integer.
% \begin{itemize}
%     \item $\transp{\vec{c}} \vec{x} = \sum_{i=1}^n c_i\cdot x_i$ is the \emph{objective function} to be minimised;
%     \item $\vec{x} \in \realn^n$ are the \emph{variables};
%     \item $\vec{c} \in \realn^n$ are the \emph{coefficients} of the variables in the objective function;
%     \item $\vec{A} \in \realn^{m \times n}$ is a matrix of coefficients such that $\vec{A} \vec{x}$ is the set of linear \emph{constraints} bounded below by $\vec{b} \in \realn^m$, which is a vector of coefficients;
%     \item $\vec{l}, \vec{u} \in (\realn \cup \infty)^n$ express lower and upper bounds of the variables; and
%     \item the index set $I \subseteq \{1, 2, \ldots, n\}$ denotes those indices $j$ where the  $x_j$ are constrained to be an integer.
% \end{itemize}

%\nd If $I= \emptyset$ then wewe speak about \emph{Linear Programming} (LP) Problems.

\nd The \emph{feasible set} of a MILP $M$ is defined as 
$M_{\mathtt{feasible}} =  \{\vec{x} \in \realn^n \mid \vec{A} \vec{x} \geq \vec{b}, \vec{l} \leq \vec{x} \leq \vec{u},  x_j \in \integern, \forall j \in I   \}$.
%
% \begin{eqnarray}\label{feasset}
% M_{\mathtt{feasible}}  & = & \{\vec{x} \in \realn^n \mid \vec{A} \vec{x} \geq \vec{b}, \vec{l} \leq \vec{x} \leq \vec{u}, \nonumber \\
% && x_j \in \integern, \forall j \in I   \} \ .
% \end{eqnarray}

We say that $M$ is \emph{feasible} if $M_{\mathtt{feasible}} \neq \emptyset$, and \emph{unfeasible} otherwise.
It is well-known that the feasibility problem, like the $k$-SAT problem, is a NP-complete problem as well.
In this work we are only interested in the feasibility/unfeasibility problem and, thus, the minimisation of the objective function is omitted from a MILP. We also recall that the \emph{MPS} file format\footnote{\url{https://lpsolve.sourceforge.net/5.5/mps-format.htm}.} is a standard way of representing and archiving MILP problems.

%Now, consider a $k$-CNF formula $\phi$ as per Eq.~\ref{kcnfset} (represented as set of sets). 
Now, consider a $k$-CNF formula $\phi$ represented as set of sets. 
We transform now the $k$-SAT problem of $\phi$ into a feasibility problem of a MILP $M_\phi$. 
The transformation follows a well-known method in transforming propositional formulae into a set of linear equations (see, \eg~\cite{Haehnle94a,Jeroslow89}). Specifically, for each propositional letter $p_i$ occurring in $\phi$, consider a binary integer variable $x_i$. The intuition here is that $x_i = 1$ (resp.~$x_i=0)$ means that $p_i$ is true (resp.~$p_i$ is false).
For a literal $l$, we use the following transformation

\begin{eqnarray}
    \tau(l) & = & \begin{cases}
                        x_i & \text{ if } l=p_i \\
                        1-x_i & \text{ if } l = \neg p_i \ . 
                \end{cases}
\end{eqnarray}

\nd A clause $C_i  =  \{l_{i1}, \ldots , l_{ik} \}$ is transformed into the linear formula

\begin{equation} \label{tauC}
  \tau(C_i) =   \sum_{j=1}^k \tau(l_{ij}) \ .
\end{equation}

\nd Eventually, the formula $\phi$ is translated into the MILP $M_\phi$:

\begin{equation}\label{milpphi} 
\begin{array}{l}
    \tau(C_1)  \geq  1 \ , \ldots,  \tau(C_m) \geq  1 \ ,
    x_i \in \{0,1\}, \forall p_i \in \Sigma_\phi \ ,
\end{array}
\end{equation}

% \begin{equation}\label{milpphi} 
% \begin{array}{rcl}
%     \tau(C_1) & \geq & 1 \\
%     &\vdots& \\ 
%     \tau(C_m &) \geq & 1 \\
%     \multicolumn{3}{l}{x_i \in \{0,1\}, \forall p_i \in \Sigma_\phi} \ ,
% \end{array}
% \end{equation}

\nd which we may succinctly represent as the MILP $M_\phi$ ($n=|\Sigma_\phi|$)

\begin{equation} \label{milpphiB}
\begin{array}{l}
    \vec{A}_\phi \vec{x}_\phi \geq \vec{b}_\phi \ , 
    \vec{x}_\phi \in \{0,1\}^n \ ,
\end{array}
\end{equation}

% \begin{equation} \label{milpphiB}
% \begin{array}{l}
%     \vec{A}_\phi \vec{x}_\phi \geq \vec{b}_\phi \ , \\
%     \vec{x}_\phi \in \{0,1\}^n \ ,
% \end{array}
% \end{equation}

\nd where $\vec{x}_\phi \in \{0,1\}^n$ is the vector of variables related to the propositional letters occurring in $\phi$, $\vec{A}_\phi \in \{-1,0,1\}^{m \times n}$ and $\vec{b}_\phi \in \{1-k, \ldots, -1,  0, 1\}^m$ are respectively, the matrix and vector of integers resulting from the encoding of the linear expressions as from Eq.~\ref{tauC}. That is, for $1\leq i \leq m$,
$b_i  =  1 - neg_i$,
%
% \begin{eqnarray}
% b_i    & = & 1 - neg_i \ , \label{biform} 
% \end{eqnarray}
where $neg_i$ is the number of negative literals in $C_i$ and, thus, $0\leq neg_i \leq k$.

The following is an easy provable result:
\begin{proposition} \label{pphi}
    Consider a $k$-CNF formula $\phi$ and let $M_\phi$ its transformation $M_\phi$ into MILP according to Eq.~\ref{milpphiB}. Then $\phi$ is satisfiable iff  $M_\phi$ is feasible.
\end{proposition}

% Also note, for instance,  if we are looking for a \emph{minimal models} of a formula $\phi$, we may encode it as adding 
% \begin{equation}
% \min \sum_{p \in \Sigma_\phi} x_p
% \end{equation}

\begin{example} \label{expropmilp}
    Consider the 2-CNF (Krom) formula $\phi = (p_1 \lor p_2) \land (p_1 \lor \neg p_2) \land (\neg p_1 \lor p_2)$, which is equivalent to  $p \land q$, \ie~$\phi \equiv p_1 \land p_2$. $\phi$'s set of set representation is
$\phi  =  \{ \{p_1, p_2 \}, \{p_1 , \neg p_2 \}, \{\neg p_1 , p_2 \} \}$.
%
% \begin{eqnarray*}
%         \phi & = & \{ \{p_1, p_2 \}, \{p_1 , \neg p_2 \}, \{\neg p_1 , p_2 \} \} \ .
%     \end{eqnarray*}
%
It can easily be verified that the MILP transformation $M_\phi$, consisting of two variables and three constraints is
$x_1 + x_2  \geq  1, x_1 - x_2  \geq  0, - x_1 + x_2  \geq  0, x_1, x_2 \in \{0,1\}$
%
% \begin{equation}\label{milpphiex} 
%     \begin{array}{rcl}
%         x_1 + x_2 & \geq & 1 \\
%         x_1 - x_2 & \geq & 0 \\
%         - x_1 + x_2 & \geq & 0 \\
%         \multicolumn{3}{l}{x_1, x_2 \in \{0,1\}}
%     \end{array}
% \end{equation}
%
\nd and, thus, we have

\begin{equation}\label{milpphiexmatrix} 
\vec{A}_\phi  = 
    \begin{bmatrix}
        1 & 1 \\
        1 & -1 \\
        -1 & 1 
    \end{bmatrix} \ , \
\vec{b}_\phi  = 
    \begin{bmatrix}
        1  \\
        0  \\
        0  
    \end{bmatrix}     \ .
\end{equation}    

\end{example}

%-----------
\subsection{From MILPs to GNNs} \label{secmilpgnn}
%-----------

\nd Now, we adapt the GNN method to solve MILP problems proposed in~\cite{Chen23,Gasse19},
%\footnote{A similar method has been proposed in~\cite{Chen23a}, where however \emph{Linear Programming} (LP) problems have been addressed instead.}  
to our specific case. It is based on a MILP representation in terms of weighted bipartite graphs, which is a quite typical approach when using MILP in connection with GNNs (see~\cite{Zhang23} for an overview).
Specifically, a MILP problem $M$ is represented as an undirected \emph{Weighted bipartite graph} $G = (V \cup W,E, \vec{X})$ with nodes in $V \cup W$ partitioned into two groups $V = \{v_1, \ldots, v_m\}$ and $W = \{w_1, \ldots, w_n\}$, with $V \cap W = \emptyset$, and $E$ is a set of weighted edges where each edge has one end in $V$ and one end in $W$. $E$ can also be seen as a function $E: V\times W \to \realn$.

Consider a finite dataset $\D$ of formulae in $\calL^k_{\mathtt{CNF}}(n,m)$ for fixed $n,m \in \naturaln$ (\ie~$k$-CNF formulae of maximal $n$ variables and $m$ clauses). Consider $\phi \in \D$ and its MILP transformation $M_\phi$ (Eq.~\ref{milpphiB}), its weighted bipartite representation  $G_\phi  = (V_\phi \cup W_\phi, E_\phi, \vec{X}_\phi)$, is as follows:
\begin{enumerate}
    \item $W_\phi$ is a set of nodes $w_j$ representing variable $x_j$ occurring in $M_\phi$ (\ie~propositional letter $p_j$ occurs in $\phi$). We call $w_j$ a \emph{variable node};
    
    \item $V_\phi$ is a set of nodes $v_i$ representing the $i$-th constraint occurring in $M_\phi$. We call $v_i$ a \emph{constraint node};

    \item the  edge connecting node constraint $v_i$ and variable node $w_j$ has weight $E_{ij} = (\vec{A_\phi})_{ij}$. Essentially, there is an edge $\tuple{v_i,w_j}$ if the variable $x_j$ occurs in the $i$-th constraint (\ie~$(\vec{A_\phi})_{ij}\neq 0$) and the weight of this edge is the value $(\vec{A_\phi})_{ij}$;

    \item the feature vector $\vec{x}_{v_i}$ of constraint node $v_i$ is $\vec{b}_i$ (so, has dimension $1$);

    \item the feature vector $\vec{x}_{w_j}$ of variable node $w_j$ is $\tuple{}$ (so, has dimension $0$).
    
\end{enumerate}
\nd We call the graph $G_\phi$ so constructed a \emph{MILP-graph}. Note that the degree of $G_\phi$ is at most $n+m$.

\begin{example}[Example~\ref{expropmilpB} cont.] \label{expropmilpB}
The bipartite graph representation of the MILP problem in Example~~\ref{expropmilp} is shown in Fig.~\ref{figexpropmilp}. It has two variable nodes $w_1,w_2$ and three constraint nodes $v_1, v_2, v_3$. There is an edge $\tuple{v_i,w_j}$ if $(\vec{A_\phi})_{ij}\neq 0$  and the weight of this edge is $(\vec{A_\phi})_{ij}$. Each constraint node $v_i$ has a feature of dimension 1, which is $\tuple{b_i}$.

\begin{figure}
\centering
\includegraphics[scale=0.35]{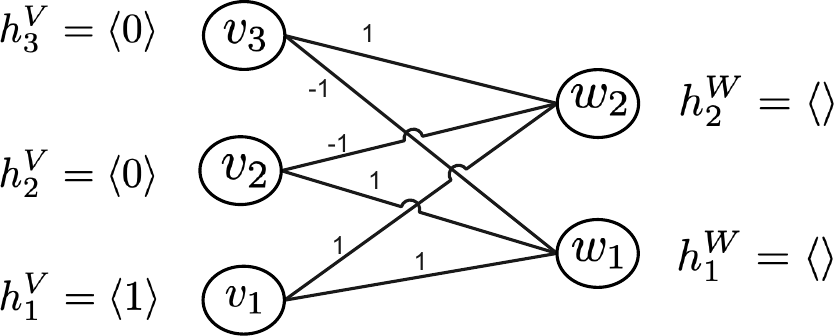}\\
\vspace*{2ex}
\caption{Bipartite graph representation of the MILP in Example~\ref{expropmilp}. }\label{figexpropmilp}
\vspace*{2ex}
\end{figure}
\end{example}

\nd Now, the GNNs in this paper always take an MILP-Graph as input, having at most $n$ variable nodes and at most $m$ constraint nodes, and the output  is a single real number used to predict the property of the whole graph (\cf~the case of graph classification), \ie~whether the underlying MILP problem (resp.~$k$-CNF formula) is feasible (resp.~satisfiable).

\begin{remark} \label{remnodepred}
Another case would be that each node in $W$ has an output and, thus, the GNN has an output in $\realn^n$, which is used to predict the property of each variable (\cf~the case of node classification). This may used in case we would like to predict also the model of a $k$-CNF formula, \ie~the truth of the variables in that model (if it exists). We will leave this task for future work.
\end{remark}

\nd Now we define the GNN structure precisely, which consist of the following steps.
For a sequence $\tuple{L_1, ..., L_l}$ of layers $L_k$:
%With $\calH^V$ (resp.~$\calH^W$) we denote the family of feature matrixes $\vec{H}$
%
\begin{description}
    \item[Input layer:]  the initial mapping at the input layer $L_1$ ($k = 1$) consists of an encoding of input feature $\vec{x}^V_i$ (resp.~$\vec{x}^W_j$)  of the $i$-th constraint node (resp.~$j$-variable node)  into the embedding space by means of learnable function $f^V_\mathtt{in}$ (resp.~$f^W_\mathtt{in}$), where $1 \leq i \leq m, 1 \leq j \leq n$:
%
% {\small
% \begin{eqnarray}
%      \vec{h}_i^{1,V} & = & f^V_\mathtt{in}(\vec{x}^V_i)  \\
%      \vec{h}_j^{1,W} & = & f^W_\mathtt{in}(\vec{x}^W_i)  \ , 
% \end{eqnarray}
% }

\begin{equation*}
    \vec{h}_i^{1,V}  =  f^V_\mathtt{in}(\vec{x}^V_i) \ , 
    \vec{h}_j^{1,W}  =  f^W_\mathtt{in}(\vec{x}^W_j)  \ , 
\end{equation*}

\nd where $\vec{h}_i^{1,V}, \vec{h}_j^{1,W} \in \realn^d$ are initial embedded constraint and variable nodes features and $d$ is their dimension;

\item[Hidden layer:] the hidden layers $L_k$ ($1 < k < l$), are updated via learnable functions $f^V_k,  f^W_k \colon \realn^d  \times \realn^d \to \realn^d$ and $g^V_k,  g^W_k \colon  \realn^d \times \realn^d  \to \realn^d$ as follows $(1 \leq i \leq m, 1 \leq j \leq n)$ :

\begin{eqnarray}
     \vec{h}_i^{k+1,V} & = & g^V_{k+1}(\vec{h}_i^{k,V},  \sum_{j=1}^n E_{ij} f^W_{k+1}( \vec{h}_i^{k,W} ) ) \label{hiddmilpA}\\
     \vec{h}_j^{k+1,W} & = & g^W_{k+1}( \vec{h}_i^{k,W}, \sum_{i=1}^m E_{ij} f^V_{k+1}( \vec{h}_j^{k,V} ) )   \label{hiddmilpB} \ .
     %\vec{h}_j^{k+1,W} & = & g^W_{k+1}( \vec{h}_i^{k,W}, \sum_{i=1}^m E_{ij} f^V_{k+1}( \vec{h}_j^{k+1,V} ) )  
\end{eqnarray}

\item[Output layer:] for the (single) graph-level output, we use a learnable function $f_\mathtt{out}  \colon \realn^d  \times \realn^d \to \realn$ that maps the variable node embedding space of the last layer $L_l$ into a value  $\hat{y}_G \in \realn$ as follows:

\begin{eqnarray}
     \hat{y}_G & = & f_\mathtt{out}(\sum_{i=1}^m \vec{h}_i^{l,V},  \sum_{j=1}^n \vec{h}_j^{l,W}) \ . \label{gnnout}
\end{eqnarray}

\end{description}
\nd All functions in the so defined GNN, namely 
$f^V_\mathtt{in}$,
$f^W_\mathtt{in}$,
$f_\mathtt{out}$,
$\{f^V_k,  f^W_k,  =g^V_k,  g^W_k\}_{k=2}^{l-1}$
are learnable functions. Moreover, please note that, in our the current implementation, there is no update of edge features, which remain constant to their initial weight. 
%
% Also note that the update rules in Eqs.~\ref{hiddmilpA} and \ref{hiddmilpB} follow a message-passing way, where each node only collects
% information from its neighbours. Since $E_{ij} = 0$ if there is no connection between constraint node $v_i$ and variable node $w_j$, 
% the sum operator in Eq~\ref{hiddmilpA} (resp.~\ref{hiddmilpB})  can be rewritten as 
% $\sum_{w_j\in N(v_i)}$ (resp.~$\sum_{w_j\in N(v_i)}$).
%

% \begin{remark} \label{milpimpl}
%     Let us note that, by inspecting the code of~\cite{Chen23a,Chen23},\footnote{See their code file {\tt models.py}, procedure  
%     {\tt forward}.} it seems to us that the authors have implemented a slightly variant of the message passing layer, \viz~Eq.~\ref{hiddmilpB}. In fact, they first update the constraint nodes embeddings and then update the variable node embeddings with the updated constraint node embeddings. That is, in Eq.~\ref{hiddmilpB}, $\vec{h}_j^{k,V}$ is replaced with $\vec{h}_j^{k+1,V}$.
% %
% %     \begin{eqnarray}
% %      \vec{h}_i^{k+1,V} & = & g^V_{k+1}(\vec{h}_i^{k,V},  \sum_{j=1}^n E_{ij} f^W_{k}( \vec{h}_i^{k,W} ) )  \ (1 \leq i \leq m) \label{hiddmilpA}\\
% %      \vec{h}_j^{k+1,W} & = & g^W_{k+1}( \vec{h}_i^{k,W}, \sum_{i=1}^m E_{ij} f^V_{k}( \vec{h}_j^{k+1,V} ) )  \ (1 \leq j \leq n) \label{hiddmilpB} \ , 
% % \end{eqnarray}
% \end{remark}

We anticipate that we will classify a formula $\phi$ as satisfiable iff $\hat{y}_{G_\phi} > 0.5$, that is the output of the learned GNN applied to the bipartite input graph $G_\phi$ is greater than $0.5$.
We call the so obtained model \emph{\loggnn}. We also define $\calF_{\mathtt{GNN}}^{n,m}$ as the set of all  mappings 
$\calF_{\mathtt{GNN}}^{n,m}  =  \{F\colon \calL^k_{\mathtt{CNF}}(n,m) \to \realn \mid  F(\phi) = \hat{y}_{G_\phi}, \forall \phi \in  \calL^k_{\mathtt{CNF}}(n,m) \}$,
that is the set of mappings $F$ that 
take as input a $k$-CNF formula $\phi \in \calL^k_{\mathtt{CNF}}(n,m)$ (as set of sets);
transform $\phi$ into the MILP problem $M_\phi$;
represent $M_\phi$ as weighted bipartite graph $G_\phi$;
give $G_\phi$ as input to the GNN; and
has $\hat{y}_{G_\phi} \in \realn$ as output.

% \begin{eqnarray} \label{fGNN}
%     \calF_{\mathtt{GNN}}^{n,m} & = & \{F\colon \calL^k_{\mathtt{CNF}}(n,m) \to \realn \mid  F(\phi) = \hat{y}_{G_\phi}, \nonumber \\
%     && \forall \phi \in  \calL^k_{\mathtt{CNF}}(n,m) \}\ ,
% \end{eqnarray}
% \nd that is the set of mappings $F$ that 
% \begin{enumerate}
%     \item take as input a $k$-CNF formula $\phi \in \calL^k_{\mathtt{CNF}}(n,m)$ (as set of sets);
%     \item transform $\phi$ into the MILP problem $M_\phi$;
%     \item represent $M_\phi$ as weighted bipartite graph $G_\phi$;
%     \item give $G_\phi$ as input to the GNN; and
%     \item has $\hat{y}_{G_\phi} \in \realn$ as output.
% \end{enumerate}

\nd Note that $G_\phi$'s order is at most $n+m$.

%-----------------
\subsection{Some Properties} \label{somep}
%-----------------

\nd In what follows, let $\Phi_{\mathtt{sat}}^{n,m}$ be the classification function
$\Phi_{\mathtt{sat}}^{n,m} \colon \calL^k_{\mathtt{CNF}}(n,m) \to \{0,1\}$,
%
% \begin{equation} \label{satmapping}
%     \Phi_{\mathtt{sat}}^{n,m} \colon \calL^k_{\mathtt{CNF}}(n,m) \to \{0,1\}   \ ,
% \end{equation}
%
where $\Phi_{\mathtt{sat}}^{n,m}(\phi) =1$ if the $k$-CNF formula $\phi \in  \calL^k_{\mathtt{CNF}}(n,m)$ is satisfiable, $\Phi_{\mathtt{sat}}^{n,m}(\phi) = 0$, otherwise.
We are going now to address the following question:
%is there a GNN able to predict whether a $k$-CNF formula $\phi$ is satisfiable?
%
% \begin{equation} \tag{$\star$}%\label{star}
%    \text{ Is there a GNN able to predict whether a $k$-CNF formula $\phi$ is satisfiable? }
% \end{equation}
%
% More specifically, does there exist $F\in \calF_{\mathtt{GNN}}^{n,m}$ that well approximates  $\Phi_{\mathtt{sat}}^{n,m}$?
does there exist $F\in \calF_{\mathtt{GNN}}^{n,m}$ that well approximates  $\Phi_{\mathtt{sat}}^{n,m}$?
Proposition~\ref{unisatprop} later on will provide an answer to this question. But, beforehand we address some properties.

%-----------------
%\subsubsection{Invariance.}\label{sectinvariance}
{\bf Invariance.}\label{sectinvariance}
Consider $\phi \in \calL^k_{\mathtt{CNF}}(n,m)$. Let
$\Sigma_\phi = \{p_1, ..., p_n\}$ be the alphabet of propositional letters occurring in $\phi$,
define $S_m = \{1, \ldots, m\}$ and  $S_n = \{1, \ldots, n\}$.
Let $\calS_m$ be the set of permutations $\tau^c$ on $S_m $ and 
let $\calS_n$ be the set of permutations $\tau^p$ on $S_n$.\footnote{A permutation of a set $S \subseteq \naturaln$ is defined as a bijection from $S$ to itself. That is, 
$\tau \colon S \to S$ is equivalent to the rearrangement of the elements of $S$ in which each element $i$
is replaced by the corresponding $\tau(i)$.}
Essentially, a permutation $\tau^c$ is a rearrangement of the clauses in $\phi$, \ie~each clause $C_i$ is replaced with $C_{\tau^c(i)}$ ($1\leq i \leq m$), while  $\tau^p$ is a rearrangement of the propositional letters occurring in $\phi$, \ie~each propositional letter $p_j$ occurring in $\phi$ is replaced with $p_{\tau^p(j)}$ ($1\leq i \leq n$). With $\tau^c(\phi)$ (resp.~$\tau^p(\phi)$) we denote the $k$-CNF formula obtained from $\phi$ by applying $\tau^c$ (resp.~$\tau^p$) to $\phi$ in the obvious way. Moreover,
for a world (interpretation) $w$ over $\Sigma_\phi$, we denote with $\tau^p(w)$ the interpretation obtained from
$w$ by applying $\tau^p$ to $w$ (\ie, each propositional letter $p_j$ occurring in $w$ is replaced with $p_{\tau^p(j)}$). If $W \subseteq W_\Sigma$ is a set of worlds, then $\tau^p(W) = \{\tau^p(w) \mid w \in W \}$.

Now, define $\calS^{m \times n} = \calS_m \times \calS_n$ and consider
$\vec{\tau} = \tuple{\tau^c , \tau^p} \in  \calS^{m \times n}$.
With $\vec{\tau}(\phi)$ we denote the $k$-CNF formula $\psi$ obtained from $\phi$ by applying $\tau^c$ and $\tau^p$ to $\phi$. 
Obviously, $\vec{\tau}(\phi) \in \calL^k_{\mathtt{CNF}}(n,m)$.
The following proposition can be shown.

\begin{proposition} \label{propmodelinvariance}
Consider $\phi \in \calL^k_{\mathtt{CNF}}(n,m)$ and let $\vec{\tau} = \tuple{\tau^c , \tau^p} \in  \calS^{m \times n}$. Then the following hold: 
\ii{i} $w$ is a model of $\phi$ iff $w$ is a model of $\tau^c(\phi)$;
\ii{ii} $w$ is a model of $\phi$ iff $\tau^p(w)$ is a model of $\tau^p(\phi)$;
\ii{iii} $w$ is a model of $\phi$ iff $\tau^p(w)$ is a model of $\vec{\tau}(\phi)$; and
\ii{iv} $\phi$ is satisfiable iff $\vec{\tau}(\phi)$ is satisfiable.
% \begin{enumerate}
%     \item $w$ is a model of $\phi$ iff $w$ is a model of $\tau^c(\phi)$;
%     \item $w$ is a model of $\phi$ iff $\tau^p(w)$ is a model of $\tau^p(\phi)$;
%     \item $w$ is a model of $\phi$ iff $\tau^p(w)$ is a model of $\vec{\tau}(\phi)$;
%     \item $\phi$ is satisfiable iff $\vec{\tau}(\phi)$ is satisfiable.
% \end{enumerate}
\end{proposition}

\nd The above proposition tells us that the satisfiability of formula is \emph{permutation invariant}, in the sense that satisfiability is preserved if we rearrange clauses, variables or both. As a consequence, we obviously have
$\Phi_{\mathtt{sat}}^{n,m}(\phi)  =  \Phi_{\mathtt{sat}}^{n,m}(\vec{\tau}(\phi)) \ , \forall \vec{\tau} \in  \calS^{m \times n}$.

Another property that we are going show is that if we make a permutation $\vec{\sigma}(\phi)$ on a formula $\phi$, the output of the GNN does not change. Formally, we say that a mapping $F\in  \calF_{\mathtt{GNN}}^{n,m}$ is \emph{permutation invariant} if it satisfies
$F(\phi)  =  F(\vec{\sigma}(\phi)),\forall \vec{\sigma} \in  \calS^{m \times n}$
% \begin{eqnarray}
%     F(\phi) & = & F(\vec{\sigma}(\phi)),\forall \vec{\sigma} \in  \calS^{m \times n} \ \label{finvariance}
% \end{eqnarray}
while we say that $F\in  \calF_{\mathtt{GNN}}^{n,m}$ is \emph{equivalence invariant} if it satisfies
$F(\phi)  =  F(\psi), \forall \phi, \psi \in \calL^k_{\mathtt{CNF}}(n,m) \text{ s.t. } \phi \equiv \psi$.
% \begin{eqnarray}
%     F(\phi) & = & F(\psi), \forall \phi, \psi \in \calL^k_{\mathtt{CNF}}(n,m) \text{ s.t. } \phi \equiv \psi \ . \label{eqinvariance}
% \end{eqnarray}
Clearly, the feasibility of a MILP problem $M$ does not change if the variables and/or constraints in $M$ are reordered. In particular, the MILP problem $M_\phi$  is feasible iff $M_{\vec{\sigma}(\phi)}$ is feasible.
A direct consequence of Proposition~\ref{propmodelinvariance} above and~\cite{Chen23} is the following stating that if we rearrange clauses and/or propositional letters the result of the GNN does not change.
%
%\begin{proposition}[Permutation invariance] \label{invariance}
\begin{proposition}\label{invariance}
    Consider $\phi \in \calL^k_{\mathtt{CNF}}(n,m)$. Let $\vec{\sigma} \in  \calS^{m \times n}$ be a permutation of clauses and/or propositional letters in $\phi$.   Then all $F\in \calF_{\mathtt{GNN}}^{n,m}$ are permutation invariant, 
    \ie~$F(\phi) = F(\vec{\sigma}(\phi))$.
\end{proposition}

% \nd Now, consider two equivalent formulae $\phi$ and $\psi$. Obviously, we have that 
% \begin{eqnarray}
%     \Phi_{\mathtt{sat}}(\phi) & = & \Phi_{\mathtt{sat}}(\psi) \ . \label{Phiinvarianceequiv}
% \end{eqnarray}

% \nd Moreover, Proposition~\ref{propequivkcnf} shown in Appendix~\ref{mdnftransform}  states that two equivalent formulae $\phi$ and $\psi$ in $k$-CNF consist both of $m$ clauses over $n$ variables, for some $n,m>0$, and are identical up to a permutation of clauses and/or variables, \ie~there is a permutation $\vec{\tau} \in  \calS^{m \times n}$ such that $\phi = \tau(\psi)$.
% Therefore, from Propositions~\ref{invariance} and ~\ref{propequivkcnf}, the following corollary follows.

% \nd Similarly, though with a slightly more involved proof, we can show equivalence invariance.\footnote{Roughly, the proof is based on the fact that two equivalent formulae $\phi$ and $\psi$ in $k$-CNF consist both of $m$ clauses over $n$ variables, for some $n,m>0$, and are identical up to a permutation of clauses and/or variables.} 

\nd Similarly, with a slightly more involved proof, it can be shown that:
% \footnote{Roughly, the proof is based on the fact that two equivalent formulae $\phi$ and $\psi$ in $k$-CNF consist both of $m$ clauses over $n$ variables, for some $n,m>0$, and are identical up to a permutation of clauses and/or variables.} 

%\begin{proposition}[Equivalence invariance] \label{invarianceequiv}
\begin{proposition}\label{invarianceequiv}
Consider two equivalent formulae $\phi$ and $\psi$ in $\calL^k_{\mathtt{CNF}}(n,m)$. Then 
$\Phi_{\mathtt{sat}}^{n,m}(\phi) = \Phi_{\mathtt{sat}}^{n,m}(\psi)$; and all $F\in \calF_{\mathtt{GNN}}^{n,m}$ are equivalence invariant, \ie~$F(\phi) = F(\psi)$.
%, for all $F\in \calF_{\mathtt{GNN}}$.
% \begin{enumerate}
%     \item $\Phi_{\mathtt{sat}}^{n,m}(\phi) = \Phi_{\mathtt{sat}}^{n,m}(\psi)$; and
%     \item all $F\in \calF_{\mathtt{GNN}}^{n,m}$ are equivalence invariant, \ie~$F(\phi) = F(\psi)$.
%     %, for all $F\in \calF_{\mathtt{GNN}}$.
% \end{enumerate}
\end{proposition}

% \mycomment{Can we extend Prop~\ref{invariance} to equivalent formulae? Specifically, can we show that any equivalent formula $\psi$ to $\phi$ can be obtained via a permutation, or in general $F(\phi)= F(\psi)$ holds, with $\Sigma_\phi = \Sigma_\psi$ ? Can we show the following: let $\phi, \psi$ be two equivalent $k$-CNF
% formulae, \ie~$\phi \equiv \psi$, with $\Sigma_\phi = \Sigma_\psi$,
% then there is $\vec{\sigma}$ such that $\psi = \vec{\sigma}(\phi)$ ? Prove it via DNF transformation of $\phi$ and $\psi$, as $\phi_{\mathtt{DNF}} \equiv \psi_{\mathtt{DNF}}$ has to hold}.

%-----------------
%\subsubsection{On GNN limitations to represent $k$-CNF formulae} \label{failuskcnf}
{\bf On GNN limitations to represent $k$-CNF formulae.}
%-----------------
We now show a limitation of GNNs to represent  $k$-CNF formulae. 
%Our results are essentially inherited from~\cite{Chen23} showing a similar limitation of GNNs to represent MILPs.
%
% To well approximate the function $\Phi_{\mathtt{sat}}^{n,m}$, $\calF_{\mathtt{GNN}}^{n,m}$ should have \emph{stronger separation power} than $\Phi_{\mathtt{sat}}^{n,m}$~\cite{Chen23}. That is, for any $k$-CNF formula $\phi$ and $\psi$ in $\calL^k_{\mathtt{CNF}}(n,m)$, 
% %
% %\begin{equation} \label{seppower}
% %    \Phi_{\mathtt{sat}}^{n,m}(\phi) \neq  \Phi_{\mathtt{sat}}^{n,m}(\psi) 
% %    \text{ implies } F(\phi) \neq F(\psi) \text{ for some } F\in \calF_{\mathtt{GNN}}^{n,m} \ .
% %\end{equation}
% %
% \begin{eqnarray} 
%     \Phi_{\mathtt{sat}}^{n,m}(\phi) \neq  \Phi_{\mathtt{sat}}^{n,m}(\psi) 
%     & \text{ implies } & F(\phi) \neq F(\psi) \nonumber \\
%     && \text{ for some } F\in \calF_{\mathtt{GNN}}^{n,m} \ . \label{seppower}
% \end{eqnarray}
%
% \nd Specifically, as long as two formulae differ \wrt~satisfiability, there should be some GNN that can detect that and give a different output. Otherwise, we way that the whole GNN family $\calF_{\mathtt{GNN}}^{n,m}$ cannot distinguish two formulae that differ \wrt~satisfiability.
% In the following, we will show that unfortunately the latter is the case.
%
Specifically, we will show that the whole GNN family $\calF_{\mathtt{GNN}}^{n,m}$ cannot distinguish two formulae that differ \wrt~satisfiability.

% It has recently been shown that the \emph{expressiveness} of GNNs can be
% related to the  \emph{Weisfeiler-Leman} (WL)
% graph isomorphism test~\cite{Weisfeiler68} and logics, such as \eg~\emph{finite variable counting logic} $\mathbf{C}_2$ (see, \eg~\cite{Barcelo20,Loukas20,Grohe21,Morris21,Morris23,Morris19,Pflueger24,Tahmasebi23,Xu19} and Appendix~\ref{WLalgo}).

It has recently been shown that the \emph{expressiveness} of GNNs can be
related to the  \emph{Weisfeiler-Leman} (WL)
graph isomorphism test~\cite{Weisfeiler68} and logics, such as \eg~\emph{finite variable counting logic} $\mathbf{C}_2$ (see, \eg~\cite{Barcelo20,Grohe21,Loukas20,Morris19,Morris21,Morris23,Pflueger24,Tahmasebi23,Xu19}).
Let us recall that the WL graph isomorphism test is a heuristic test for the existence of an \emph{isomorphism} between two graphs $G$ and $H$. Specifically, in graph theory, an \emph{isomorphism} of graphs $G$ and $H$ is a bijection $f$ between the nodes of $G$ and $H$, such that any two nodes $v_1$ and $v_2$ of $G$ are adjacent in $G$ if and only if $f(v_1)$ and $f(v_2)$ are adjacent in $H$. We say that $f$ is a \emph{edge preserving} bijection.
If an isomorphism exists between two graphs, then the graphs are called \emph{isomorphic}, and is denoted $G \cong H$. 
%
% In the case when the isomorphism is a mapping of a graph onto itself, \ie, when $G$ and $H$ are one and the same graph, the isomorphism is called an \emph{automorphism} of $G$.
%
For \emph{labelled graphs}, we consider the following definition.
A \emph{labelled graph} is of the form  $G = \tuple{V,E,l}$, where $l$ is a 
\emph{label function} $l\colon V\cup E \to S$, where $S \subseteq \naturaln$. $l(r)$ is a \emph{label} of $r \in V \cup E$. 
In the case of labelled graphs, an isomorphism is a node bijection $f$ which is both edge-preserving and label-preserving, that is,  we additionally require that $l(w) = l(f(w))$ for $w \in V \cup E$.

% In the case of the MILP encoding as bipartite graphs (see Appendix~\ref{milpsbip}),   a variant of the 1-WL test has been proposed in~\cite{Chen23a,Chen23}. This algorithm has been further adapted  to our context as illustrated in the algorithm $\mathtt{WL}_{\mathtt{MILP}}$ below.
% In the algorithm,

The algorithm $\mathtt{WL}_{\mathtt{MILP}}$ below is a variant of the WL (\viz.~1-WL) algorithm adapted to our case.\footnote{See~\cite{Chen23} for similar algorithms for the MILP case.}
In the algorithm, $\mathtt{Colour}_{0}^{V}$ (resp.~$\mathtt{Colour}_{0}^{W}$)  injectively maps the label of a constraint (resp.~variable) node to a unique natural number; and $\mathtt{Colour}_{i}^{V}$ and $\mathtt{Colour}_{i}^{W}$ injectively map a  pair to a unique natural number, which has not been used in previous iterations.
%
% In the case of the MILP encoding as bipartite graphs (see Appendix~\ref{milpsbip}),   a variant of the 1-WL test has been proposed in~\cite{Chen23a,Chen23}. This algorithm has been further adapted  to our context as illustrated in the algorithm $\mathtt{WL}_{\mathtt{MILP}}$ below.
% In the algorithm,
% \begin{itemize}
%     % \item $\mathtt{Colour}_{0}^{V}$ injectively maps the above pair to a unique natural number, which has not been used in previous iterations
%     \item $\mathtt{Colour}_{0}^{V}$ (resp.~$\mathtt{Colour}_{0}^{W}$)  injectively maps the label of a constraint (resp.~variable) node to a unique natural number;
%
%     \item $\mathtt{Colour}_{i}^{V}$ and $\mathtt{Colour}_{i}^{W}$ injectively map a  pair to a unique natural number, which has not been used in previous iterations.   
% \end{itemize}
%
%
\begin{algorithm}[t]
%\caption{1-WL test for MILP-Graphs (denoted by WL$_{\text{MILP}}$)}
\caption{$\mathtt{WL}_{\mathtt{MILP}}$} \label{alg:wltest}
{%\footnotesize
\scriptsize
\begin{algorithmic}[1]
\Require A bipartite labelled graph $G = \tuple{V\cup W, E, l}$ 
\Ensure A colouring of the nodes in $G$
\State $d \gets |V\cup W|$
\Comment{$d$ is the degree of $G$}
\ForAll{$v\in V$} 
    \State $C_0^V(v) \gets \mathtt{Colour}_{0}^{V} (l(v))$     
    \Comment{Initialise constraint node colour}
\EndFor
\ForAll{$w\in W$}
    \State $C_0^W(w) = \mathtt{Colour}_{0}^{W} (l(w))$ %for all $w\in W$ 
    \Comment{Initialise variable node colour}
    \EndFor
\For{$i = 1,2, \dots, d$}
    \ForAll{$v\in V$} 
        \State $C_i^V(v) \gets \mathtt{Colour}_{i}^{V}(C_{i-1}^V(v),  \multiset{\tuple{C_{i-1}^{W}(w), E_{vw}} \mid w \in N(v)})$
        \Comment{Update constraint node colour}
    \EndFor    
    \ForAll{$w\in W$}
        \State $C_i^W(w) \gets \mathtt{Colour}_{i}^{W}(C_{i-1}^W(w),  \multiset{\tuple{C_{i-1}^{V}(v), E_{vw}} \mid v \in N(w)})$
        \Comment{Update variable node colour}
    \EndFor   
\EndFor
\State \Return $\tuple{C_{d}^V, C_d^W}$
\end{algorithmic}
}
\end{algorithm}
%
% , following the usual schema:
% each node iteratively updates its colour, based on its own colour and information from its neighbours.
%
After $i$ iterations the colour of a node $v$ captures some structure of its $k$-hop neighbourhood, \ie~the subgraph induced by all nodes reachable by walks of length at most $k$. 
To test if two graphs $G$ and $H$ are non-isomorphic, we run the above algorithm in parallel on both graphs. If the two graphs have a different number of nodes coloured $c \in \naturaln$ at some iteration, the 1-WL concludes that the graphs are not isomorphic.
Now, as any $k$-CNF formula $\phi$ is encoded as a labelled bipartite graph $G_\phi$, the above colouring of a labelled graph induces a colouring of propositional letters and clauses in $\phi$.
as illustrated by Algorithm~$\mathtt{WL}_{\mathtt{kCNF}}$.

\begin{algorithm}
%\caption{1-WL test for MILP-Graphs (denoted by WL$_{\text{MILP}}$)}
\caption{$\mathtt{WL}_{\mathtt{kCNF}}$} \label{alg:wltestsat}
{%\footnotesize
\scriptsize
\begin{algorithmic}[1]
\Require A $k$-CNF formula $\phi$ 
%bipartite labelled graph $G = \tuple{V\cup W, E, l}$ 
\Ensure A colouring of propositional letters and clauses in $\phi$
\State $G_\phi \gets \tuple{V_\phi\cup W_\phi, E_\phi, l_\phi}$ 
\Comment{Initialise $G_\phi$ as the labelled bipartite graph representation of $\phi$}
\State \Return $\mathtt{WL}_{\mathtt{MILP}}(G_\phi)$
\end{algorithmic}
}
\end{algorithm}

We can now use this colouring to define a sort of isomorphism among formulae. Specifically, given two $k$-CNF formulae $\phi$ and $\psi$, we say that $\phi$ and $\psi$ are \emph{graph-isomorphic} iff their labelled bipartite graph representations $G_\phi$ and $G_\psi$ are isomorphic.

Unfortunately, there exist some non-isomorphic graph pairs that WL test fail to distinguish. That is, it is well-known that the WL test is known to be correct, but not complete, in the sense that if the WL algorithm says that two graphs are not isomorphic, then they are indeed so. However, the other direction does not hold: that is, there are non-isomorphic graphs where the difference is not detected by the WL algorithm. Therefore, it is not surprising that this applies to the case of $k$-CNF formulae as well. 
    
Specifically, throughout this paper, we use the notation $\phi \sim \psi$ to denote that $\phi$ and $\psi$ \emph{cannot be distinguished by the 1-WL test}, that is, $\mathtt{WL}_{\mathtt{kCNF}}(\phi) = \mathtt{WL}_{\mathtt{kCNF}}(\psi)$, for all colouring functions $\mathtt{Colour}_{0}^{V}$, $\mathtt{Colour}_{0}^{W}$, $\mathtt{Colour}_{i}^{V}$ and $\mathtt{Colour}_{j}^{W}$ ($i,j>0$).

The following proposition indicates that $\calF_{\mathtt{GNN}}^{n,m}$ actually has the same expressive power as the 1-WL test, which is a direct consequence of~\cite[Theorem 3.1]{Chen23} 
%(see also~\cite[Theorem 4.2]{Chen23a}) 
proven for arbitrary MILP problems.

\begin{proposition}[Theorem 3.1 in~\cite{Chen23}]\label{thm4.2chen23}
For all $\phi,\psi \in \calL^k_{\mathtt{CNF}}(n,m)$, it holds that $\phi \sim \psi$ iff $F(\phi) = F(\psi)$, for all $F \in\calF_{\mathtt{GNN}}^{n,m}$.
\end{proposition}

\nd Unfortunately, the following proposition reveals that the 1-WL test is less expressive than $\Phi_{\mathtt{sat}}^{n,m}$.
%and, consequently, GNN has weaker separation power than $\Phi_{\mathtt{sat}}^{n,m}$, on some specific $k$-CNF datasets.

\begin{proposition}  \label{undist}
    There are two $k$-CNF formulae $\phi$ and $\psi$, the former satisfiable while the latter is not, but $\phi \sim \psi$, \ie~they cannot be distinguished by the 1-WL test. 
\end{proposition} 
\begin{proof}
In the following, let us define the \emph{exclusive or} operator as follows:
$\varphi \lxor \varphi' \defined (\varphi \lor \varphi') \land  (\neg \varphi \lor \neg \varphi')$.
% {\small
% \[
% \varphi \lxor \varphi' \defined (\varphi \lor \varphi') \land  (\neg \varphi \lor \neg \varphi') \ .
% \]
% }
\nd Now, let $\phi$ and $\psi$ defined as follows

\begin{eqnarray}
    \phi & = & (p_1 \lxor p_2) \land (p_2 \lxor p_3) \land (p_3 \lxor p_4) \land (p_4 \lxor p_5) \nonumber \\
    && \land (p_5 \lxor p_6) \land (p_6 \lxor p_1) \label{phyind} \\
    \psi & = & (p_1 \lxor p_2) \land (p_2 \lxor p_3) \land (p_3 \lxor p_1) \land (p_4 \lxor p_5) \nonumber \\
    && \land (p_5 \lxor p_6) \land (p_6 \lxor p_4) \label{psiind}
\end{eqnarray}

\nd One may easily verify that $\phi$ is satisfiable ($w=p_1\bar{p}_2p_3\bar{p}_4p_5\bar{p}_6$ is model), while $\psi$ is not.
%
% \nd We next show that $\phi$ is satisfiable, while $\psi$ is not.
% Concerning $\phi$ one can verify that indeed 
% $w=p_1\bar{p}_2p_3\bar{p}_4p_5\bar{p}_6$ is model of $\phi$. Concerning $\psi$,
% assume by contradiction that there is a model $w$ of $\psi$. Then, either $w\sat p_1$  or $w\not \sat p_1$. 
% \begin{description}
%     \item[Case $w\sat p_1$.] As $w\sat p_1 \lxor p_2$,
% $w\not \sat p_2$ has to hold. But, $w\sat p_2 \lxor p_3$, so $w\sat p_3$ follows and, thus, as $w\sat p_3 \lxor p_1$, we get $w\not \sat p_1$, a contradiction;
%
%     \item[Case $w\not \sat p_1$.] As $w\sat p_1 \lxor p_2$,
% $w \sat p_2$ has to hold. But, $w\sat p_2 \lxor p_3$, so $w\not \sat p_3$ follows and, thus, as $w\sat p_3 \lxor p_1$, we get $w \sat p_1$, a contradiction.
% \end{description}
%
Therefore, $\Phi_{\mathtt{sat}}(\phi) \neq \Phi_{\mathtt{sat}}(\psi)$.
Now, we show however that $\phi \sim \psi$, \ie~$\mathtt{WL}_{\mathtt{kCNF}}(\phi) = \mathtt{WL}_{\mathtt{kCNF}}(\psi)$. That is, $\phi$ and $\psi$ cannot be distinguished by the 1-WL test.

The MILP representations $M_\phi$ and $M_\psi$ are illustrated below ($a_h$, with $h=1,2$, are constants with $a_1=1,a_2=-1$):

\begin{equation}\label{milpphiexA}
M_\phi \defined \left \{
\begin{array}{cc}
\begin{array}{l}
   a_hx_i + a_hx_{i+1}  \geq a_h  \\
   a_hx_6 + a_hx_1  \geq a_h
\end{array} &
\begin{array}{l}
  1\leq i < 6,  x_j \in \{0,1\}, \\
   1\leq j \leq 6
\end{array}  
%\multicolumn{2}{c}{(1\leq i < 6), x_j \in \{0,1\}, (1\leq j \leq 6)} 
\end{array} \right .
\end{equation}

% \begin{equation}\label{milpphiexA}
% M_\phi \defined \left \{
% \begin{array}{cc}
% \begin{array}{rcl}
%    x_1 + x_2 & \geq 1 \\
%    x_2 + x_3 & \geq 1 \\
%    x_3 + x_4 & \geq 1 \\
%    x_4 + x_5 & \geq 1 \\
%    x_5 + x_6 & \geq 1 \\
%    x_6 + x_1 & \geq 1
% \end{array} &
% \begin{array}{rcl}
%    - x_1 - x_2 & \geq -1 \\
%    - x_2 - x_3 & \geq -1 \\
%    - x_3 - x_4 & \geq -1 \\
%    - x_4 - x_5 & \geq -1 \\
%    - x_5 - x_6 & \geq -1 \\
%    - x_6 - x_1 & \geq -1
% \end{array}  \\
% \multicolumn{2}{c}{x_i \in \{0,1\}} 
% \end{array} \right .
% \end{equation}
%

\begin{equation}\label{milppsiex}
M_\psi \defined \left \{
\begin{array}{cc}
\begin{array}{l}
   a_hx_i + a_hx_{i+1}  \geq a_h  \\
   x_3 + x_1  \geq a_h\\
   a_h x_k + a_h x_{k+1}  \geq a_h  \\
   x_6 + x_4  \geq a_h
\end{array} &
\begin{array}{l}
   1\leq i < 3, 4\leq k < 6 \\
   x_j \in \{0,1\}, 1\leq j \leq 6
\end{array}  
%\multicolumn{2}{c}{(1\leq i < 3), (4\leq k < 6) x_j \in \{0,1\}, (1\leq j \leq 6)} 
\end{array} \right .
\end{equation}

% \begin{equation}\label{milppsiex}
% M_\psi \defined \left \{
% \begin{array}{cc}
% \begin{array}{rcl}
%    x_1 + x_2 & \geq 1 \\
%    x_2 + x_3 & \geq 1 \\
%    x_3 + x_1 & \geq 1 \\
%    x_4 + x_5 & \geq 1 \\
%    x_5 + x_6 & \geq 1 \\
%    x_6 + x_4 & \geq 1
% \end{array} &
% \begin{array}{rcl}
%    - x_1 - x_2 & \geq -1 \\
%    - x_2 - x_3 & \geq -1 \\
%    - x_3 - x_1 & \geq -1 \\
%    - x_4 - x_5 & \geq -1 \\
%    - x_5 - x_6 & \geq -1 \\
%    - x_6 - x_4 & \geq -1
% \end{array}  \\
% \multicolumn{2}{c}{x_i \in \{0,1\}} 
% \end{array} \right .
% \end{equation}
%
\nd It can easily verified that indeed, $M_\phi$ is feasible 
($\vec{\bar{x}} = \tuple{1,0,1,0,1,0}$ is a solution), while $M_\psi$ is not.
\begin{figure}[t]
\centering
\begin{subcaptionblock}{.25\textwidth}
\centering
\includegraphics[scale=0.25]{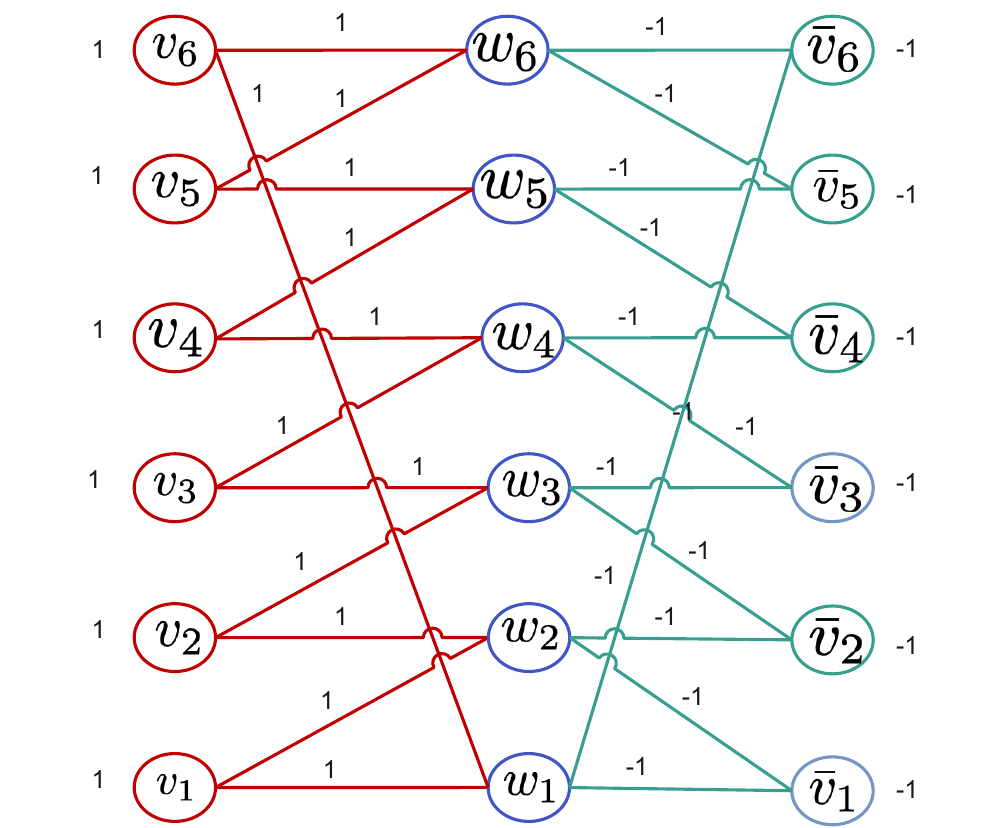}
\caption{Bipartite graph $G_\phi$.}\label{figbgphi}
\end{subcaptionblock}%
\begin{subcaptionblock}{.25\textwidth}
\centering
\includegraphics[scale=0.25]{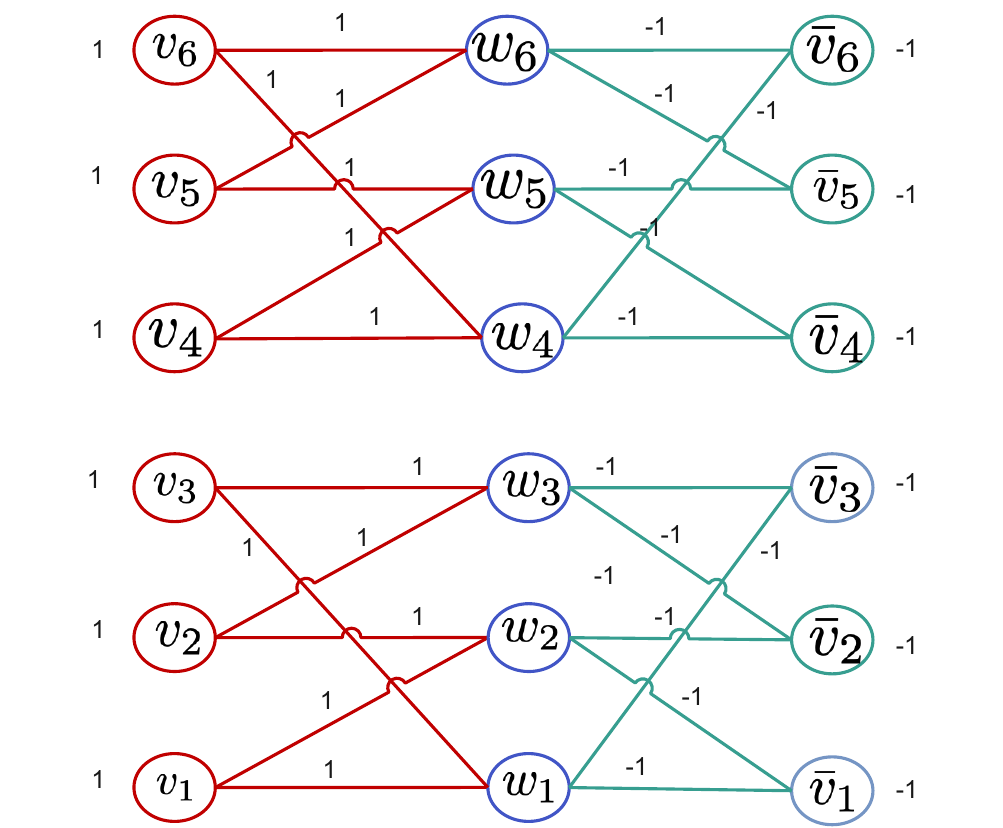}
\caption{Bipartite graph $G_\psi$.}\label{figbgpsi}
\end{subcaptionblock}%
\vspace*{2ex}
\caption{Bipartite graph representations of $\phi$ and $\psi$ of Proposition~\ref{undist}.}\label{figNNphipsi}
\vspace*{3ex}
\end{figure}

In Fig.~\ref{figNNphipsi} we illustrate the labelled bipartite graph representations $G_\phi$ (Fig.~\ref{figbgphi}) and $G_\psi$ (Fig.~\ref{figbgpsi}) of $\phi$ (\viz~$M_\phi$) and $\psi$ (\viz~$M_\psi$), respectively ($w_j$ are node variables, $v_i$ and $\bar{v}_i$ are node constraints).
%
% In these graphs,  the red labelled constraint nodes and edges have weight/label $1$, while the green labelled constraint nodes and edges have weight/label $-1$.
% In $G_\phi$, the node variables $w_j$ ($1\leq j \leq 6$)  represent the variables $x_j$, the node constraint $v_i$ ($1\leq i \leq 5$) represents the equation
% $x_i + x_{i+1} \geq 1$, while the node constraint $\bar{v}_i$ ($1\leq i \leq 5$) represents the equation $- x_i - x_{i+1} \geq -1$. $v_6$ (rep.~$\bar{v}_6$) represents $x_6 + x_1 \geq 1$ (resp.~$-x_6 - x_1 \geq -1$). 
% The case for $G_\psi$ is similar.

Now, we show that the two graphs in Fig.~\ref{figNNphipsi} can not be distinguished by WL test, which can be proved by induction.
In fact, it is not difficult to see that we may end up with a colouring as illustrated in Fig.~\ref{figNNphipsi}: 
each constraint node $v_i$ ($1\leq i \leq 6$), labelled with $1$, is connected to $2$ variable nodes $w_i$ ($1\leq i \leq 6$) via  edges labelled $1$; and 
each constraint node $\bar{v}_i$ ($1\leq i \leq 6$), labelled with $-1$, is connected to $2$ variable nodes $w_i$ ($1\leq i \leq 6$) via  edges labelled $-1$.

% In fact, with a similar argument as in~\cite{Chen23}, 
% it is not difficult to see that we may end up with a colouring as illustrated in Fig.~\ref{figNNphipsi}: 
% \begin{itemize}
%     \item each constraint node $v_i$ ($1\leq i \leq 6$), labelled with $1$, is connected to $2$ variable nodes $w_i$ ($1\leq i \leq 6$) via  edges labelled $1$;

%     \item each constraint node $\bar{v}_i$ ($1\leq i \leq 6$), labelled with $-1$, is connected to $2$ variable nodes $w_i$ ($1\leq i \leq 6$) via  edges labelled $-1$.
    
% \end{itemize}

\nd Consequently, one cannot distinguish the two graphs by checking the outputs of the 1-WL test. 
\end{proof}

\nd The above proposition tells us that there are two formulae $\phi$ and $\psi$ such that
$\Phi_{\mathtt{sat}}^{n,m}(\phi) \neq \Phi_{\mathtt{sat}}^{n,m}(\psi)$ (for some $n,m$), but $\phi \sim \psi$ and, thus, by Proposition~\ref{thm4.2chen23},  $F(\phi) = F(\psi)$, for all $F \in\calF_{\mathtt{GNN}}^{n,m}$.

% \nd The above proposition tells us that there are two formulae $\phi$ and $\psi$ such that
% $\Phi_{\mathtt{sat}}^{n,m}(\phi) \neq \Phi_{\mathtt{sat}}^{n,m}(\psi)$ (for some $n,m$), but $\phi \sim \psi$. By Proposition~\ref{thm4.2chen23},  
%  $F(\phi) = F(\psi)$, for all $F \in\calF_{\mathtt{GNN}}^{n,m}$ and, thus, Eq.~\ref{seppower} is not satisfied.

%-----------------
%\subsection{GNNs with RNI for  $k$-CNFs} \label{randfeat}
{\bf GNNs with RNI for  $k$-CNFs.}
%-----------------
\nd We next show how one may easily get rid of the above mentioned shortcomings.  To do so, one may rely on so-called GNNs with \emph{Random Node Initialisation} (RNI), 
which allows to increasing the expressiveness of GNNs~\cite{Abboud20,Grohe21,Sato21}. Not surprisingly, this technique has also been used in fact in~\cite{Chen23} for MILP problems, which we illustrate shortly here.
Specifically,  given graph $G = \tuple{V,E,\vec{X}}$, the idea of RNI consists of extending each node's feature node $\vec{x}_v$ with an additional \emph{random value},  \eg~drawn \emph{uniformly} from the interval $[0,1]$. 
% Note that from a theoretical point of view, the exact distribution is not important. However, in practice, it has some effect, see~\cite{Abboud20}. 
The intuition is that by appending an additional random feature, each node will have distinct features and the resulting bipartite graph can be distinguished by the 1-WL algorithm. 
%(see Appendix~\ref{wlgnnRI}).
%
% the idea of RNI consists of extending each node's initial embedding 
% $\vec{h}_v^1$ (see Eq.~\ref{gnnbasic}) with an additional \emph{random value}  say, drawn \emph{uniformly} from the interval $[0,1]$.\footnote{Note that from a theoretical point of view, the exact distribution is not important. However, in practice, it has
% some effect, see~\cite{Abboud20}.}
%
In the following, we call such GNNs as \emph{GNNs with RNI}. 
and define $\calF_{\mathtt{RNIGNN}}^{n,m}$ as the set of all  mappings $\calF_{\mathtt{RNIGNN}}^{n,m} = \{F\colon \calL^k_{\mathtt{CNF}}(n,m) \to \realn \mid  F(\phi) = \hat{y}_{G_\phi}, \forall \phi \in  \calL^k_{\mathtt{CNF}}(n,m) \}$
that rely on GNNs with RNI. 
% \begin{eqnarray} \label{fRNIGNN}
%     \calF_{\mathtt{RNIGNN}}^{n,m} & = & \{F\colon \calL^k_{\mathtt{CNF}}(n,m) \to \realn \mid  F(\phi) = \hat{y}_{G_\phi}, \nonumber \\
%     &&  \forall \phi \in  \calL^k_{\mathtt{CNF}}(n,m) \}\ ,
% \end{eqnarray}
%
% \nd that is the set of mappings $F$ that 
% \begin{enumerate}
%     \item take as input a $k$-CNF formula $\phi \in \calL^k_{\mathtt{CNF}}(n,m)$ (as set of sets);
%     \item transform $\phi$ into the MILP problem $M_\phi$;
%     \item represent $M_\phi$ as weighted bipartite graph $G_\phi$;\footnote{$G_\phi$'s order is still at most $n+m$.}
%     \item \emph{give $G_\phi$ as input to the GNN with RNI}; and
%     \item has $\hat{y}_{G_\phi} \in \realn$ as output.
% \end{enumerate}
%
As a consequence, the output $\hat{y}_G$ (see Eq.~\ref{gnnout}) of GNN with RNI is no longer deterministic, but becomes a random variable.
Nevertheless, as explained in~\cite{Chen23}, GNNs with RNI still remain permutation invariant in the sense that  the outcome of the computation of  GNN with RNI does not depend on the specific representation of the input graph. Consequently, 
Proposition~\ref{invariance} (permutation invariance) and Proposition~\ref{invarianceequiv} (equivalence invariance) continue to hold for $\calF_{\mathtt{RNIGNN}}^{n,m}$ as well.

The following proposition states that, by adding random features to the weighted bipartite graph representation of a formula $\phi$ provides sufficient power to represent $\phi$'s satisfiability status. That is,
let $\satindicator^{n,m}:\calF_{\mathtt{RNIGNN}}^{n,m} \times \calL^k_{\mathtt{CNF}}(n,m) \to \{0,1\}$ be the binary classification function 
defined as 

\begin{equation} \label{satindicator}
\satindicator^{n,m}(F,\phi) = 
\begin{cases}
    1 & \text{if } F(\phi) > \frac{1}{2} \\
    0 & \text{otherwise } .
\end{cases}
\end{equation}

\nd Specifically, given a GNN with RNI $F \in \calF_{\mathtt{RNIGNN}}^{n,m}$ and a $k$-CNF formula $\phi \in \calL^k_{\mathtt{CNF}}(n,m)$, $F$ predicts $\phi$ to be satisfiable if the output $\hat{y}_{G_\phi}$ of $F$ applied to $\phi$ is above $\frac{1}{2}$. If $\satindicator^{n,m}(F,\phi) \neq \Phi_{\mathtt{sat}}^{n,m}(\phi))$ then we say that $F$ \emph{misclassified} $\phi$.
Then using~\cite[Theorem 5.1]{Chen23}, the following \emph{Universal SAT Approximation Property} can be shown.

%\begin{proposition}[Universal $k$-SAT Approximation Property] \label{unisatprop}
\begin{proposition}\label{unisatprop}
Consider a finite dataset $\D$ of formulae in $\calL^k_{\mathtt{CNF}}(n,m)$. For all $\epsilon > 0$, there exists 
$F \in \calF_{\mathtt{RNIGNN}}^{n,m}$ such that for all $\phi \in \D$,
$Pr(\satindicator^{n,m}(F,\phi) \neq     \Phi_{\mathtt{sat}}^{n,m}(\phi) ) < \epsilon$.
% {\small
% \begin{equation} \label{unisatapproxprop}
%    Pr(\satindicator^{n,m}(F,\phi) \neq     \Phi_{\mathtt{sat}}^{n,m}(\phi) ) < \epsilon \ .
% \end{equation}
% }
\end{proposition}

% \nd In other words, the above proposition states that given a finite dataset $\D$ of formulae in $\calL^k_{\mathtt{CNF}}(n,m)$, for all $\epsilon >0$ there is a GNN with RNI such that the probability of a misclassification of a formula $\phi \in \D$ is smaller than $\epsilon$, which is in line with what proved in~\cite{Abboud20,Chen23,Grohe21,Sato21}.

\nd In other words, the proposition asserts that for any finite dataset $\D$ of formulas in $\calL^k_{\mathtt{CNF}}(n,m)$ and any $\epsilon > 0$, there exists a GNN with RNI such that the misclassification probability for any formula $\phi \in \D$ is less than $\epsilon$, which is consistent with the results shown in~\cite{Abboud20,Chen23,Grohe21,Sato21}.

% Let us note that how to choose the random feature in training is significant. 
% In practice, one may generate one random feature vector $\omega \in [0,1]^m \times [0,1]^n$ for the whole dataset. 
% This setting may lead to efficiency in training GNN models, but the trained GNNs cannot be applied to datasets with formulae of of different sizes. Another practice is sampling several independent random features for each formula, but one may suffer from
% difficulty in training. %How to balance the trade-off in practice will be an interesting future topic.
% In~\cite{Abboud20} various RNI methods have been compared such as the standard normal distribution
% $\mathcal{N}(0,1)$, the uniform distribution over $[-1,1]$, the Xavier normal, the Xavier uniform distribution, and also their
% partial RNI, \ie~not all nodes get a random value, but are initialised deterministically. \cite{Chen23} uses one uniform distribution for all examples as their dataset consists of MILP problems having all the same number of variables and constraints.

%-----------------
%\subsection{GNNs for Unfoldable $k$-CNFs} \label{unfoldable}
{\bf GNNs for Unfoldable $k$-CNFs.}
%-----------------

\nd In this section we show that we may get an analogous result as Proposition~\ref{unisatprop} for GNNs without RNI, provided that we restrict the dataset to some specific formulae, called \emph{unfoldable} formulae.
Specifically, consider a formula $\phi \in \calL^k_{\mathtt{CNF}}(n,m)$
and let us considering the colouring of propositional letters and clauses of $\phi$ computed by $\mathtt{WL}_{\mathtt{kCNF}}(\phi)$.
We say that $\phi$ is \emph{foldable} if there are two clauses or two propositional letters in $\phi$ having the same colour. Otherwise, we call $\phi$ \emph{unfoldable}. For instance, both formulae $\phi$ and $\psi$ in Eqs.~\ref{phyind} and~\ref{psiind}, respectively, are foldable, while the formula $\phi$ in Example~\ref{expropmilp} is not. Note that our definition follows that of~\cite[Definition 4.1]{Chen23}, adapted to formulae.

Now, by using~\cite[Theorem 4.2.]{Chen23}, we may prove the analogue of Proposition~\ref{unisatprop} for AC-GNNs, provided that we restrict the dataset to unfoldable formulae. That is,

% \begin{proposition}[Universal $k$-SAT Approximation Property for Unfoldable Formulae] \label{unisatpropunfoldable}
\begin{proposition} \label{unisatpropunfoldable}
Consider a finite dataset $\D$ of unfoldable formulae in $\calL^k_{\mathtt{CNF}}(n,m)$. 
For all $\epsilon > 0$, there exists  $F \in \calF_{\mathtt{GNN}}^{n,m}$ such that for all $\phi \in \D$
$Pr(\satindicator^{n,m}(F,\phi) \neq     \Phi_{\mathtt{sat}}^{n,m}(\phi) ) < \epsilon$.
% {\small
% \begin{equation} \label{unisatapproxprop}
%    Pr(\satindicator^{n,m}(F,\phi) \neq     \Phi_{\mathtt{sat}}^{n,m}(\phi) ) < \epsilon \ .
% \end{equation}
% }
\end{proposition}

\nd It essentially states that for unfoldable formulae we do not need the extra expressive power of RNI. This concludes this section.

%%%%%%%%%%%%%%%%%%%%%%%%%%%%%%%%%%%%%%%%%%%%%%%%%%%%%%%%%%%%
\section{Evaluation} \label{eval}
%%%%%%%%%%%%%%%%%%%%%%%%%%%%%%%%%%%%%%%%%%%%%%%%%%%%%%%%%%%%

% The $k$-SAT problem exhibits the well-known \emph{easy-hard-easy} phase transition pattern~\cite{Crawford96,Mitchell92}, considering the execution time compared with the  \emph{clauses-to-variables ratio} $r = \frac{m}{n}$, where $m,n$ are the number of clauses and variables, respectively. The hard area is determined approximately when $m \approx 4.5 \cdot n = h$. For $r \leq h$, \ie~formulae with few clauses are under-constrained, \ie~have many models and finding one of them can be done in short time. On the other hand,  for $r \geq h$, \ie~formulae with many clauses, are over-constrained and, thus, have likely no model and its unsatisfiability can be figured out in short time, too. 
% The most difficult part is the hard area in which essentially a formula has very few or exactly one model and finding it can be time consuming. Therefore, current GNN SAT solver focuses on hard problems, and \loggnn~is no exception.

\nd The $k$-SAT problem exhibits the well-known \emph{easy-hard-easy} phase transition~\cite{Crawford96,Mitchell92}, based on the \emph{clauses-to-variables ratio} $r = \frac{m}{n}$, where $m$ and $n$ are the number of clauses and variables, respectively. The hardest instances occur near $m \approx 4.5 \cdot n = h$. For $r \leq h$, formulae are under-constrained, with many models and short solving times. For $r \geq h$, they are over-constrained, typically unsatisfiable, and also quickly resolved. The hardest cases lie near $r = h$, where formulae have few or a unique solution, making them computationally demanding. As such,  
%GNN-based SAT solvers, including \loggnn, focus on this challenging regime.
current GNN SAT solvers often focuses on hard problems, as we do here as well.

%\paragraph{Datasets and Parameters.} 
We employ here two types of SAT dataset generators. The first one is  G4SATBench \cite{Li24}, which, generates, among others, $3$-SAT datasets for GNN-based SAT solvers focussing on hard problems. G4SATBench produces sat/unsat balanced random $3$-SAT formulae at the phase-transition region (hard region), $h' = 4.258n+58.26n^{2/3}$~\cite{Crawford96}. The other one is our own  one.
%\footnote{As many readers may know too well, using others code may require a lot of work and, thus, we preferred developed our own one.}
As SAT is NP-complete, building a labelled-dataset can be very time-consuming. A peculiarity of our dataset generator is that, given  $k$, maximal number of variables $n$, maximal clause length $m$ and dataset size $s$, we generate a labelled dataset 
of $s$ random $k$-CNF formulae not exceeding the given parameters.
%\footnote{We used \url{https://massimolauria.net/cnfgen/} as $k$-CNF generator.} 
Then, during the training phase, we may specify, according to some parameters, which formulae to select (always balanced sat/unsat), among which we have \eg~the training, validation and test sizes, the fraction $f_p$ of hard formulae in the area $[h' - f_l \cdot h', h' + f_r \cdot h']$, where $f_l, f_r \in [0,1]$ are two parameters allowing to define the extend of the hard area.
We run our experiments on eight $3$-SAT datasets, 4 generated by G4SATBench and 4 generated by \loggnn, of sizes 2K and 20K,  each split into \emph{easy} (number of variables $n \in [10,40]$) and \emph{medium} ($n \in [40,200]$), in \loggnn~we have $f_l = f_r = 0.01$ and $f_p=1.0$ (all formulae are in the hard area). The validation and test size is $10\%$ of the  dataset size. The results are reported in Table~\ref{tabresults2}.
%
% \mbox{ \ } \\
% \mycomment{maybe also 3D graph, x=num.clauses, y=variables, z=percentage instances. draw the hard area}
%
%\paragraph{Measures.}
%\paragraph{Parameters.}
%
On all eight datasets we run our \loggnn~solver\footnote{The source code of \loggnn~and the $k$-SAT generator can be found here: \url{https://www.umbertostraccia.it/cs/ftp/milp-sat-gnn.arxive.zip}} and the NeuroSAT solver~\cite{Selsam19}\footnote{In~\cite{Li24} reported among the best one for $3$-SAT.} as comparison, where the latter has been executed as indicated and implemented in the G4SATBench source code.\footnote{Among others, with 32 hidden layers, 128 embedding size, learning rate $10^{-4}$} For \loggnn~we made some runs with varying hyperparameters. Given the limited resources we had access to, the number of epochs has been fixed to $150$. Concerning RNI, we considered the following percentage of nodes that have a random feature: $0.0\%$  (no RNI), $50\%$ (half of the nodes have a random feature) and $100\%$ (all nodes have a random feature). Concerning the embedding size $d$ and the number of layers $l$, we considered  $d \in \{32,64,128\}$ and $l \in \{2, 16,32\}$, respectively.  The batch size is fixed to 64 and the learning rate is $10^{-4}$. We considered \emph{Mean Squared-Error} (MSE) and  \emph{Binary Cross Entropy} (BCE) loss functions (see, \eg~\cite{Hamilton20,Wu22}), but our results showed lower effectiveness with MSE.
The only a notable exception for MSE is however the case of (d-l)=128-32 with no RNI of the G4Satbench, case 20K-easy dataset, for which we found an accuracy of 0.84.

%All experiments have been run on a Linux machine with 

% Our best results are reported in Table~\ref{tabresults}.

% \begin{table}%[t]
%     \caption{Experimental results (accuracy).} \label{tabresults}
% {\tiny
%     \begin{center}
%     \begin{tabular}{|c|c|c|c|c||c|c|c|c|c|} \hline
%      & \multicolumn{8}{c|}{Datasets} \\ 
%      & \multicolumn{4}{c||}{G4SATBench} & \multicolumn{4}{c|}{\loggnn} \\ \hline
%     & \multicolumn{2}{c|}{2000} & \multicolumn{2}{c||}{20000}  
%     & \multicolumn{2}{c|}{2000} & \multicolumn{2}{c|}{20000} \\
%     Solver      & easy  & medium    & easy  & medium    
%                 & easy  & medium    & easy  & medium    \\ \hline
%     NeuroSAT    &     0.80  &    0.75   &       &    
%                 &           &           &       &         \\ \hline
%     \loggnn     &     0.76     &   0.71        &        &           
%                 &    0.73     &           &    0.83    &  \\ \hline
%     \end{tabular}
%     \end{center}
% }
% \end{table}

% \begin{figure*}[t]
% \centering
% \begin{subcaptionblock}{.45\textwidth}
% \begin{center}
% {\tiny
%     \input{tabresults2.tex}
% } % tiny
% \end{center}
% \caption{Experimental results (accuracy, BCE).}\label{tabresults2}
% \end{subcaptionblock}%
% \hspace*{1ex}
% \begin{subcaptionblock}{.45\textwidth}
% \begin{center}
% {\tiny
%     \input{tabresults2-mse.tex}
% } % tiny
% \end{center}
% \caption{Experimental results (accuracy, MSE).}\label{tabresults2-mse}
% \end{subcaptionblock}%
% \vspace*{2ex}
% \caption{BExperimental results.}\label{results}
% \vspace*{4ex}
% \end{figure*}

\begin{table}[t]
\caption{Experimental results (accuracy, where loss function is BCE).} \label{tabresults2}
\begin{center}
{\tiny

\begin{tabularx}{\columnwidth}{p{3mm}p{5mm}p{2mm}cccccccc}
    %\toprule
    \cmidrule{1-11}
     & & & \multicolumn{8}{c}{Datasets - Loss BCE} \\ 
     & & & \multicolumn{4}{c}{G4SATBench} & \multicolumn{4}{c}{\loggnn} \\ 
     \cmidrule(lr){4-7} \cmidrule(lr){8-11} 
    & & & \multicolumn{2}{c}{2K} & \multicolumn{2}{c}{20K} & 
        \multicolumn{2}{c}{2K} & \multicolumn{2}{c}{20K} \\
    \cmidrule(lr){4-5} \cmidrule(lr){6-7} \cmidrule(lr){8-9} \cmidrule(lr){10-11}  
    Solver  & d-l & RF    & easy  & medium    & easy  & medium    
                & easy  & medium    & easy  & medium    \\
    %\midrule
    \cmidrule{2-11}
    % -------------------------------------------------------
    % Neuro\-SAT & - & -   &     .87 (Wrong) &    0.75; .73 (W); R  &    .80 (W)   &      
                &           &           &       &         \\
    Neuro\-SAT & - & -   &  .80   &    .75  &  .89     &   .79     
                &     .80      &     .77      &   .89    &   .79      \\
    % GCN      & - & -   &  .59   &  .60    &   .82    &    .76    
    %             &          &          &       &        \\
    % GGNN     & - & -   &  .50 (?)  &  0.50 (?)   &       &        
    %             &          &          &       &        \\

    % GIN     & - & -   &  .82  &  .65   &   .88    &  0.65      
    %             &          &          &       &        \\
    
    % -------------------------------------------------------
    \cmidrule{2-11}
    % -------------------------------------------------------
    % \multirow{3}{*}{\parbox{3mm}{\loggnn}} & 8-2 & 0    &  &  &  &           
    %             &   .70       &           &        &  \\
    % ---
    %     & 8-4 & 0    &  &  &  &           
    %             &   .69       &           &        &  \\
    % % --- 
    % & 16-2 & 0    &  &  &  &           
    %             &   .70       &           &        &  \\
    % -------------------------------------------------------
    % ------------------RF=0.0-------------------------------
    % -------------------------------------------------------
    \multirow{3}{*}{\parbox{3mm}{\loggnn}} 
    & 32-2 &  0   & .67 & .73 &  .68 &  .71         
        &   .71       &    .73       &   .68     &  .72 \\
    % --- 
    & 32-16 &     & .70 & .68 & .78  &    .75       
        &   .73       &   {\bf .75}        &   .78     &  {\bf .75} \\
    % --- 
    & 32-32 &     & .74 & .74 & {\bf .79} &     .75      
        &   .70       &    .73       &   .78     &  .74 \\
    \cmidrule{2-2}
    % --- 
    & 64-2 &     & .70 & {\bf .75} & .67 &   .71        
        &  .70       &      .70     &   .68     &  .72 \\
    % --- 
    & 64-16 &     & .75 & .66 & .77 & .75          
        &  .70  &   .73       &   .77     &  .74 \\
    % --- 
    & 64-32 &     & {\bf .79} & .64 & .76 &    .75       
        &  .69       & .71          &   .77    &  {\bf .75} \\
    \cmidrule{2-2}
    % --- 
    & 128-2 &     & .67 & .71 & .66 &   .71        
        &  .70        &   .70        &   .69     & .71  \\
    % --- 
    & 128-16 &     & .72 & .69 & .75 &  .75         
        &  .74        &      .70     &   .75     &  .72 \\
    % --- 
    & 128-32 &     & .73  & .67 & .75 ({\bf .84}) &  .75         
        &  .71        &       .70    &   .73     &  .74 \\
    \cmidrule{2-11}
    % -------------------------------------------------------
    % ------------------RF=0.5-------------------------------
    % -------------------------------------------------------
    & 32-2 & .5    & .71 & .73 & .67 &  .71         
        &   .69     &     .70      &   .67     & .71 \\
    % --- 
    & 32-16 &     & .73 & .69 & {\bf .79} &  .74         
        &  .69        &      {\bf .75}     &  {\bf .79}      &  {\bf .75}  \\
    % --- 
    & 32-32 &     & .76 & .66 & {\bf.79} &  .75         
        &   .72       &   .73        &   .77     &  .74 \\
    \cmidrule{2-2}
    % -------------------------------------------------------
     & 64-2 &     &  .61  & .73          &  .68      &  .71         
            &  .73       &   .67        &   .69     &  .71 \\
     & 64-16 &     &  .73  & .70          &   .77     &    .75       
            &   .72      &    .70       &   .77    & {\bf .75} \\
     & 64-32 &     &    .71     &    .66       &   .77     & {\bf .76}           
            &  .70       &     .72      &    .76   &  .74\\
     \cmidrule{2-2}                
        & 128-2 &     &  .70        &    .69       &   .67     &   .71        
            & .68        &      .71     &  .68    &  .71 \\
     & 128-16 &     &    .77      &     .69      &   .76     &   .74        
            & .70        &      .72     &   .75    &  .73 \\
     & 128-32 &     &   .74       &    .75       &   .76     &   .75        
            & .72        &      .74     &    .75    &  .73 \\
   \cmidrule{2-11}                
    % -------------------------------------------------------
    % ------------------RF=1.0-------------------------------
    % -------------------------------------------------------
       & 32-2 & 1.0    & .69 & .72 & .68 &   .71        
        &  .72      &    .70       &   .68     &  .73 \\
    % --- 
    & 32-16 &     & .75 & .65 & {\bf .79} &   .74        
        &  .73        &     .64      &   .78     &  .74 \\
    % --- 
    & 32-32 &     & .77 & .73 & .78 &   .74        
        &   .74       &     .72      &    .78    &  .73 \\
    \cmidrule{2-2}
    % -------------------------------------------------------
     & 64-2 &     &    .68      &    .73       &  .68      &    .71       
            &  .69         &   .73        &  .68      &  .72 \\
    % ---
     & 64-16 &     &   .76       &  .73         &   .76     &    .75       
            &  .73       &   .67        &  .78     & .74 \\
    % ---
     & 64-32 &    &     .71     &   .71        &   .77     &  .74           
            &  .68       &     .73      &    .76    &  {\bf .75} \\
     \cmidrule{2-2}                
        & 128-2 &    &    .70      &   .74        &  .68      &   .71        
                & .68        &    .69       &  .67    &  .70 \\
    % ---
     & 128-16 &     &    .76      &   {\bf .75}        &   .75     &   .74        
                &  .70       &   .73        &   .75    &  {\bf .75} \\
    % ---
     & 128-32 &     &     .70     &    .74       &   .76     &     .75      
                &  {\bf.75}       &      .71     &  .76      &  {\bf .75} \\
   %\cmidrule{2-10}                
    % -------------------------------------------------------   
    %\bottomrule
    \cmidrule{1-11}
\end{tabularx}
} % tiny
\end{center}
\end{table}

% \begin{table}%[t]
% \caption{Experimental results (accuracy, MSE).} \label{tabresults2-mse}
% \begin{center}
% {\tiny
%     \input{tabresults2-mse.paper.tex}
% } % tiny
% \end{center}
% \end{table}

%\paragraph{(Short) Discussion.} 
As one may see, our preliminary effectiveness results  of \loggnn, despite its current simple architecture, are promising and somewhat comparable to the one of NeuroSat\footnote{Not surprisingly, \eg~for the case 128-32 on the G4Satbench 2000 easy dataset, the average learning time per epoch of NeuroSat (70sec) is around $4.5$ times longer than the one for \loggnn (15sec).} with a more notable difference for the cases 20K-easy. Concerning \loggnn, the case (d-l) =  32-16 with $50\%$ RNI has more winning conditions. For the considered datasets, having RNI indeed has slight more winning columns than not having RNI, compliant with the results reported \eg~\cite{Abboud20,Chen23}.

%  {\scriptsize
% MEMENTO \newline
 \begin{tabular}{l}
 % easy/new.MSE\_mean\_10000\_1.0\_150\_64\_32 = 0.814 \\
 % easy/new.MSE\_mean\_10000\_0.5\_150\_64\_32 = 0.803 \\
 % easy/new.MSE\_mean\_10000\_0.0\_150\_64\_32 = 0.803 \\\\
 % redo 10000 (should be 20000). Data was wrong ! \\ \\
%easy/new.MSE\_mean\_20000\_1.0\_150\_64\_32 = 0.8335 \\
%easy/new.MSE\_mean\_20000\_0.5\_150\_64\_32 = 0.825 \\
%easy/new.MSE\_mean\_20000\_0.0\_150\_64\_32 = 0.8255 \\ 
% easy/new.BCE\_mean\_20000\_1.0\_150\_64\_32 = pc-straccia running \\ 
% easy/new.BCE\_mean\_20000\_0.5\_150\_64\_32 = pc-straccia running \\ 
% easy/new.BCE\_mean\_20000\_0.0\_150\_64\_32 = pc-straccia running \\ 
 % easy/new.MSE\_mean\_2000\_1.0\_150\_64\_32= 0.685 \\
 % easy/new.MSE\_mean\_2000\_0.5\_150\_64\_32= 0.72 \\
 % easy/new.MSE\_mean\_2000\_0.0\_150\_64\_32= 0.73 \\\\
 % g4sat.easy.BCE\_meanggr\_2000\_0.5\_150\_64\_32 =  0.76 \\
 % g4sat.new.medium\_MSE\_mean\_2000\_0.0\_150\_64\_32 =  0.69\\
 % g4sat.new.medium\_MSE\_mean\_2000\_0.5\_150\_64\_32 =  0.69 \\
 % g4sat.new.medium\_MSE\_mean\_2000\_1.0\_150\_64\_32 =  0.71 \\ \\
 % redo g4sat. Data was wrong  \\ 
 %new.medium\_MSE\_mean\_1.0\_150\_64\_32 =  pc-straccia run
 \end{tabular}
 %}

%%%%%%%%%%%%%%%%%%%%%%%%%%%%%%%%%%%%%%%%%%%%%%%%%%%%%%%%%%%%
%\section{Related Work} \label{relw}
%%%%%%%%%%%%%%%%%%%%%%%%%%%%%%%%%%%%%%%%%%%%%%%%%%%%%%%%%%%%

%%%%%%%%%%%%%%%%%%%%%%%%%%%%%%%%%%%%%%%%%%%%%%%%%%%%%%%%%%%%
\section{Conclusions} \label{concl}
%%%%%%%%%%%%%%%%%%%%%%%%%%%%%%%%%%%%%%%%%%%%%%%%%%%%%%%%%%%%

{\bf Summary.} This paper introduced MILP-SAT-GNN, a novel method that enables Graph Neural Networks (GNNs) to solve Boolean SAT (Satisfiability) problems by leveraging GNNs used for Mixed Integer Linear Programming (MILP) via a transformation of $k$-CNF formulae into MILPs, which are then encoded as weighted bipartite graphs and processed by a GNN. We contributed with some theoretical properties of the method (invariance, approximate soundness and completeness) and performed also an experimental validation showing promising results and comparable to best performing solver, despite the simple GNN architecture adopted.

{\bf Future work.} One line, obviously, relates to increase the complexity of GNN architecture with the objective to improve the effectiveness. Also, we are going to address the satisfiability assignment problem within \loggnn~(find a model if it exists). 
Moreover, we would like also to investigate how to train our system on a dataset covering the whole spectrum of formulae with varying sizes $k$, number of variables and number of clauses, to generalise the system to any input.
Another line, by noting that Proposition~\ref{pphi}, with an appropriate and more involved transformation into MILP,  may  be extended to some \emph{many-valued}~\cite{Haehnle94a} and paraconsistent~\cite{Abe15} logics, relates to the development of GNNs for these logics as well. Additionally, using MILP allows us also to combine logical constraints with  numerical constraints, such as  ``the price is $\geq 15k$". We also would like to investigate whether we can directly encode also some probabilistic logic~\cite{LukasiewiczT98,deRaedt07,Martires24}  and used in a neuro-symbolic approach such as~\cite{Manhaeve18} (see, \eg~\cite{Marra24} for an overview) in which, rather that having two separated systems (so-called, system1-system2 model), we aim at encoding both as a neural network. 

% Another line, by noting that Proposition~\ref{pphi}, with an appropriate and more involved transformation into MILP,  can be extended to any propositional formula~\cite{Jeroslow89}. Moreover, it can also be extended to various \emph{many-valued} propositional formula~\cite{Haehnle01} and \emph{propositional fuzzy logic} formula \eg~under \L ukasiewicz semantics~\cite{Straccia13}, which may hint to apply our method to many-valued and fuzzy logics as well. We will leave this for future work. A probabilistic setting may also of interest.

\clearpage

%%%%%%%%%%%%%%%%%%%%%%%%%%%%%%%%%%%%%%%%%%%%%%%%%%%%%%%%%%%%%%%%%%%%%%%%

%%% Use this environment to include acknowledgements (optional).
%%% This will be omitted in doubleblind mode.

\section*{acknowledgments}
\nd This research was partially supported by the FAIR (Future Artificial Intelligence Research) project funded by the NextGenerationEU program within the PNRR-PE-AI scheme (M4C2, investment 1.3, line on Artificial Intelligence). This work has also been partially supported by the H2020 STARWARS Project (GA No. 101086252), a type of action HORIZON TMA MSCA Staff Exchanges.

%%%%%%%%%%%%%%%%%%%%%%%%%%%%%%%%%%%%%%%%%%%%%%%%%%%%%%%%%%%%%%%%%%%%%%%%

%%%%%%%%%%%%%%%%%%%%%%%%%%%%%%%%%%%%%%%%%%%%%%%%%%%%%%%%%%%%%%%%%%%%%%%%

%%% Use this command to include your bibliography file.

%{\scriptsize
% \bibliographystyle{plain}
% \bibliography{mybiblio}
%}

\end{document}